\documentclass[11pt]{article}
\usepackage[T1]{fontenc}
\usepackage[preprint]{tmlr}
\usepackage[utf8]{inputenc}
\usepackage{amsmath,amssymb,amsthm,mathtools,bm}
\usepackage{booktabs,multirow,graphicx}
\usepackage{xcolor}
\definecolor{editorialorange}{RGB}{190,90,0}

\DeclareRobustCommand{\editorialflag}[1]{#1}
\usepackage{algpseudocode}
\usepackage{float}
\floatstyle{ruled}
\newfloat{algorithm}{htbp}{loa}
\floatname{algorithm}{Algorithm}
\usepackage{hyperref}
\hypersetup{hidelinks}
\usepackage{url}
\usepackage{enumitem}
\usepackage{microtype}
\graphicspath{{./}}
\newtheorem{proposition}{Proposition}
\newtheorem{definition}{Definition}
\newtheorem{lemma}{Lemma}
\newtheorem{theorem}{Theorem}

\title{Learning to Harvest Without Collapse in a Regenerative Commons: A Lagrangian Framework}
\author{\name Jose Tupayachi \\
\addr Oak Ridge National Laboratory \\
\addr Department of Industrial and Systems Engineering, University of Tennessee, Knoxville
\AND
\name Xueping Li \\
\addr Department of Industrial and Systems Engineering, University of Tennessee, Knoxville
\AND
\name Soham Das \\
\addr Department of Industrial and Systems Engineering, University of Tennessee, Knoxville}
\begin{document}
\maketitle

\begin{abstract}
The tragedy of the commons poses a multi-agent safety problem: reward-seeking agents can deplete a shared resource, and cooperation among its users does not itself specify how much must be preserved. We make preservation an explicit requirement by formulating a regenerative commons as a constrained Markov game or a constrained multi-agent MDP with a designer-specified depletion budget. We develop a nonstationary Lagrangian framework that constructs a policy sequence from solutions of unconstrained games or cooperative control problems. Extending earlier time-average constructions, we introduce average-epoch solution concepts for reset episodes with discounted rewards and terminal costs. We prove a reward-independent feasibility certificate, cooperative feasibility and approximate optimality against feasible policy mixtures, and an extension to unbiased sampled costs. For self-interested agents, a constrained Nash certificate quantifies the price-dispersion term introduced by deviations that redistribute budget across epochs. Under the stated assumptions on solver accuracy and multiplier updates, these results give constrained policy-sequence guarantees using solutions of unconstrained problems. Experiments with constrained IPPO and MAPPO in a Gordon--Schaefer fishery examine how depletion budgets shape stock retention, harvest rewards, and price adaptation. 
\end{abstract}


\section{Introduction}\label{sec:introduction}

Agents can pursue their assigned rewards while collectively degrading the resource on which they depend. The tragedy of the commons~\citep{hardin1968tragedy} makes this problem concrete for multi-agent AI safety. In a renewable fishery, extraction provides an immediate benefit to the harvester, while depletion changes the opportunities available to every participant. Classical commons models~\citep{gordon1954economic,levhari1980fish}, sequential social dilemmas and common-pool MARL~\citep{leibo2017multiagent,perolat2017commons}, and GovSim's language-agent societies~\citep{piatti2024collapse} motivate studying this tension in learning populations. Cooperative AI and multi-agent-risk research motivate examining collective outcomes~\citep{dafoe2020cooperative,hammond2025risks}. We ask how to make resource preservation an explicit requirement of multi-agent learning.

We begin with a designer-specified depletion budget, which sets how much common stock may be depleted by the end of an episode. We formulate the commons as a \emph{constrained Markov game} (CMG)~\citep{altman2000games} when harvesters pursue individual rewards and as a \emph{constrained multi-agent MDP} (CM-MDP) when they optimize a team reward. Both formulations impose the same ecological requirement. Cooperation among harvesters can change how extraction is coordinated, but does not itself specify the stock that must remain. The budget makes that requirement explicit, separating participant reward from protection of the shared resource~\citep{hammond2025risks}. Changing the budget changes the admissible level of depletion.

The constraint creates a learning problem. We want to use unconstrained reward-optimization algorithms to meet a depletion budget while retaining either individual incentives or a team objective. We develop a \emph{nonstationary Lagrangian policy-sequence framework}. At each epoch, a shared multiplier prices depletion and defines an unconstrained Lagrangian game or control problem. We compute and execute a stationary solution, then update the price from measured budget excess. This separates solving the reward-based problem at a given price from regulating resource use across epochs. Suitable unconstrained solvers can therefore be reused: the theory requires approximate Nash equilibria in the game or approximately optimal joint policies in the team problem, with their solution errors entering the guarantees.

The nonstationary policy sequence is the solution produced by this framework. A policy selected at one price need not itself meet the budget, and changing prices may repeatedly select different harvesting policies. Lower-depletion epochs can compensate for higher-depletion epochs, allowing the sequence to meet the resource requirement without converging to a single policy or price. The repeated process of selecting, executing, and updating policies thus provides a policy framework for the constrained problem. We choose an \emph{average-epoch} solution concept to evaluate it: expected terminal depletion is averaged across reset episodes, while each episode uses one stationary policy on the augmented state, including episode time.

Our theory extends the time-average policy-sequence constructions of~\citet{das2025cmg,das2025lagrangian}. Rewards are discounted within each episode; episodic rewards and terminal costs are then averaged across epochs. This preserves the nonstationary construction while replacing physical-time averages along continuous play with discounted episodic performance and an average terminal-resource requirement. For every realized sequence of costs, a bound links cumulative budget excess to multiplier growth independently of reward optimization. For the team problem, strict feasibility and approximate Lagrangian optimization yield multiplier control, finite-sequence feasibility and reward bounds against feasible mixtures, and an almost-sure feasibility extension for unbiased sampled costs. For the game, unilateral deviations may redistribute their budget across epochs while opponents' policies remain fixed. We bound these gains and identify the additional price-dispersion term needed for a constrained Nash certificate. The results distinguish what is needed to satisfy the budget, achieve high team reward, and limit unilateral improvement.

In the commons, the budget specifies the stock-preservation requirement, and price feedback guides the choice of harvesting policies that share this budget across episodes. We evaluate this mechanism in a Gordon--Schaefer commons~\citep{gordon1954economic,schaefer1957} using individual-reward IPPO and team-reward MAPPO. Stock, harvest, and multiplier trajectories show how the chosen budget affects resource use, reward, and price adaptation. PPO--PI learners show near-target depletion in several regimes. The test-bed illustrates resource regulation through changing policies, while the theory establishes conditions under which unconstrained subproblem solutions yield the required constrained policy sequence.

\subsection{Related work and position}\label{sec:related}

Commons research studies how individual extraction incentives interact with shared resource dynamics, from open-access fisheries to strategic harvesting~\citep{gordon1954economic,levhari1980fish}. Ostrom's institutional account emphasizes that resource governance admits diverse arrangements~\citep{ostrom1990governing}. MARL studies sequential social dilemmas~\citep{leibo2017multiagent} and common-pool appropriation~\citep{perolat2017commons}; GovSim examines sustainable cooperation among language agents~\citep{piatti2024collapse}. Cooperation mechanisms include inequity aversion, reciprocity, and learned incentives~\citep{hughes2018inequity,eccles2019learning,yang2020incentivize}. In renewable fisheries, shared environmental signals support temporal harvesting conventions and sustainable resource use~\citep{danassis2022signals}. Our formulation specifies a depletion budget and develops guarantees for the policy sequences generated by its Lagrangian updates.

Feedback-evolving games study how strategic behavior and a shared environmental resource jointly evolve. \citet{weitz2016oscillating} characterize oscillating tragedies of the commons, while \citet{paarporn2018optimal} study interventions through incentives and environmental information. \citet{gavin2026learning} examine how boundedly rational learning and incentive policies affect resource preservation, and \citet{vahmian2026tragedy} analyze strategic extraction across multiple populations. Our formulation evaluates policy sequences across reset episodes under an explicit average terminal-depletion constraint, with guarantees for cooperative reward and constrained unilateral deviations.

Constrained MDPs formalize reward optimization subject to resource requirements~\citep{altman1999cmdp}. Duality results give conditions under which constrained RL can be addressed through Lagrangian optimization~\citep{paternain2019duality}; obtaining a policy that meets the constraints remains a separate issue. Averaging the solutions obtained during dual-subgradient updates can recover feasible solutions~\citep{larsson1999ergodic,nedic2009primal}. Constrained RL already returns mixtures of policies generated by reward-optimization solvers~\citep{le2019batch,miryoosefi2019convex}; state-augmented RL incorporates evolving multipliers into execution~\citep{calvo2024state}. Our construction uses a solver to select a stationary policy at each price and evaluates the resulting policy sequence across reset episodes.

The closest predecessors are the Lagrangian policy-sequence frameworks for constrained Markov games and cooperative MARL of~\citet{das2025cmg,das2025lagrangian}. They execute stationary solutions of unconstrained Lagrangian problems in sequence and analyze reward and cost averaged over continuous play. We retain that method and develop an average-epoch criterion with discounted episodic rewards and terminal costs. The cooperative analysis compares against feasible policy mixtures and treats unbiased sampled costs. The strategic analysis allows unilateral deviations to redistribute budget across epochs, exposing a price-dispersion term in the Nash certificate. The contribution concerns the sequence criterion and its guarantees, building on established Lagrangian and policy-mixture constructions. Appendix~\ref{app:source-results} details this relationship. A different Lagrangian approach by~\citet{ding2023generalized} proves regret and constraint-violation bounds for episodic two-player zero-sum games with independent transitions. Our general-sum result bounds feasibility and unilateral improvement in terms of solver accuracy and the generated prices.

Practical constrained-policy methods include CPO~\citep{achiam2017constrained} and cooperative safe MARL~\citep{gu2023safe}. ACPO addresses average-reward constrained MDPs~\citep{agnihotri2024acpo}, whose physical-time average differs from our average across reset episodes. Our experiments use PI feedback motivated by~\citet{stooke2020responsive}, while the theorems assume the specified projected multiplier update and bounds on the solvers' solution errors. The experiments illustrate resource regulation with PPO--PI learners. Appendix~\ref{app:extended-context} expands the governance and constrained-learning context, including comparisons with structured constrained-game methods.

\section{A Regenerative Commons and Its Solution Concepts}\label{sec:formulation}

We consider $n=2$ harvesters with shared biomass $B_t\in[0,K]$. Each episode begins at $B_0=K$ and lasts $H$ decision steps. Agent $i$ chooses effort $e_{i,t}\in[0,1]$ and receives catch $h_{i,t}=q e_{i,t}B_t$; if requested aggregate catch exceeds biomass, catches are scaled proportionally. Writing $B'_t=B_t-\sum_i h_{i,t}$, the Gordon--Schaefer transition~\citep{gordon1954economic,schaefer1957} is
\begin{equation}\label{eq:stock}
B_{t+1}=\operatorname{clip}_{[0,K]}\!\left[B'_t+rB'_t\left(1-\frac{B'_t}{K}\right)\right].
\end{equation}
The stage reward is raw catch $R_{i,t}=h_{i,t}$. The experimental protocol specifies $K=1000$, $r=0.3$, $q=0.5$, $H=60$, and $\gamma=0.99$ (Appendix~\ref{app:protocol}).

Let $\mathcal P_i$ be the chosen class of stationary within-episode policies $\pi_i:s\mapsto\Delta(\mathcal A_i)$ on the specified augmented observation (including episode time) and $\mathcal P=\prod_i\mathcal P_i$. For $\boldsymbol\pi\in\mathcal P$, define discounted episodic reward and terminal depletion
\begin{equation}\label{eq:episode-values}
J_i^\gamma(\boldsymbol\pi)=\mathbb E_{\boldsymbol\pi}\!\left[\sum_{t=0}^{H-1}\gamma^t R_{i,t}\right],\qquad
C(\boldsymbol\pi)=\mathbb E_{\boldsymbol\pi}\!\left[1-\frac{B_H}{K}\right].
\end{equation}
The requirement $C\leq\varepsilon$ demands an expected terminal stock of at least $(1-\varepsilon)K$.  The two objectives differ: in the constrained Markov game (CMG), agent $i$ maximizes $J_i^\gamma$ against the other agent, whereas in the constrained multi-agent MDP (CM-MDP), a designer maximizes $J_\Sigma^\gamma=\sum_iJ_i^\gamma$ subject to the same global budget.

For the commons, average-epoch feasibility expresses a resource requirement across repeated harvesting episodes: expected terminal stock, averaged across episodes, must be at least $(1-\varepsilon)K$. This permits different harvesting policies to share a depletion budget, rather than requiring each policy selected during price adjustment to meet it separately. The requirement is appropriate when resource use may be balanced across episodes that reset to the same initial stock. We therefore evaluate the entire policy sequence. For $M$ successive episodes, let $\boldsymbol\Pi_M=(\boldsymbol\pi^0,\ldots,\boldsymbol\pi^{M-1})$, where each $\boldsymbol\pi^k$ is stationary within episode $k$ but the selected policy may change between episodes. With the episode reset just specified, its average-epoch values are
\begin{equation}\label{eq:epoch-values}
\bar J_{i,M}(\boldsymbol\Pi_M)=\frac1M\sum_{k=0}^{M-1}J_i^\gamma(\boldsymbol\pi^k),\qquad
\bar C_M(\boldsymbol\Pi_M)=\frac1M\sum_{k=0}^{M-1}C(\boldsymbol\pi^k).
\end{equation}
For example, with budget $\varepsilon=0.1$, two policies with expected depletion $0.05$ and $0.15$, used equally often, meet the average budget and retain $90\%$ of carrying capacity on average at episode end. The second policy exceeds the budget individually. Thus the requirement controls average terminal depletion, not stock loss in every episode or at every physical time step.

Without resets, changing a policy alters future starting-state distributions and the right-hand sides of~\eqref{eq:epoch-values} must be conditioned on each realized starting state. The earlier CMG and CM-MDP results instead evaluate continuous-play time averages of the form $V_i^s(\Pi)=\liminf_{T\to\infty}T^{-1}\mathbb E_\Pi[\sum_{t<T}R_i(s_t,a_t)\mid s_0=s]$ and analogous average constraints~\citep{das2025cmg,das2025lagrangian}. Our $k$ indexes reset episodes and not physical time in one uninterrupted trajectory.

\begin{definition}[Average-epoch constrained Nash target]\label{def:epoch-nash}
A sequence $\boldsymbol\Pi_M$ is an $(\alpha,\delta)$-average-epoch constrained Nash profile if $\bar C_M(\boldsymbol\Pi_M)\leq\varepsilon+\alpha$ and, for every agent $i$ and every sequence of stationary unilateral deviations $\boldsymbol\Sigma_i=(\sigma_i^0,\ldots,\sigma_i^{M-1})$ satisfying $M^{-1}\sum_k C(\sigma_i^k,\boldsymbol\pi_{-i}^k)\leq\varepsilon$,
\begin{equation}\label{eq:epoch-nash}
\bar J_{i,M}(\boldsymbol\Pi_M)\geq
\frac1M\sum_{k=0}^{M-1}J_i^\gamma(\sigma_i^k,\boldsymbol\pi_{-i}^k)-\delta.
\end{equation}
\end{definition}
This is a constrained generalized Nash criterion~\citep{rosen1965concave,facchinei2007generalized} on the product of players' policy-sequence spaces: each player's feasible deviations depend on the opponents. The target is episodic and discounted, and an alternative policy sequence can reallocate its constraint budget across episodes. We hold opponents' policy sequence fixed when testing a deviation. An equilibrium of the adaptive learning algorithm, in which a deviation also changes future prices, would be a different object. Requiring alternative policies to satisfy $C(\sigma_i^k,\boldsymbol\pi_{-i}^k)\leq\varepsilon$ at every epoch gives a weaker test of equilibrium, which can be strictly weaker when deviations redistribute resource use across epochs.

With reset episodes, a public uniform draw of one shared policy index reproduces the sequence's average values. This may improve over requiring every component to be feasible, but need not improve over unrestricted randomized stationary policies. Appendix~\ref{app:environment} gives the precise distinction. For the cooperative arm, the corresponding benchmark is
\begin{equation}\label{eq:opt-coop}
J_{\Sigma,M}^{\star}=\sup_{\boldsymbol\Pi_M:\,\bar C_M(\boldsymbol\Pi_M)\leq\varepsilon}
\frac1M\sum_{k=0}^{M-1}J_\Sigma^\gamma(\boldsymbol\pi^k).
\end{equation}

The team formulation asks how effectively coordinated harvesters can use the commons while meeting its preservation requirement. A shared objective allows policies to be evaluated by their total harvest and compared with the best policy sequence satisfying the same depletion budget. Cooperation aligns the harvesters' objectives, while the explicit constraint specifies how much stock must remain on average at episode end. The game formulation instead retains individual incentives and asks whether any harvester can improve its own reward through an average-feasible deviation. The two formulations therefore address complementary questions: the quality of coordinated resource use and the strategic stability of individually motivated resource use. Writing $\bar J_{\Sigma,M}=\sum_i\bar J_{i,M}$, a feasible sequence is $\delta$-optimal when its average team value is at least $J_{\Sigma,M}^{\star}-\delta$.

\section{Lagrangian Policy Sequences and Guarantees}\label{sec:method}

The framework constructs policy sequences by adapting the mechanism of the average-reward CMG and cooperative results~\citep{das2025cmg,das2025lagrangian}. For a shared budget and price $\lambda\geq0$, define the per-agent and team Lagrangians
\begin{equation}\label{eq:lagrangians}
\mathcal L_i(\boldsymbol\pi,\lambda)=J_i^\gamma(\boldsymbol\pi)-\lambda[C(\boldsymbol\pi)-\varepsilon],\qquad
\mathcal L_\Sigma(\boldsymbol\pi,\lambda)=J_\Sigma^\gamma(\boldsymbol\pi)-\lambda[C(\boldsymbol\pi)-\varepsilon].
\end{equation}
At price $\lambda_k$, an ideal CMG oracle returns a stationary $\delta$-Nash profile of the unconstrained Lagrangian game with payoffs $\mathcal L_i(\cdot,\lambda_k)$; an ideal CM-MDP oracle returns a $\delta$-optimal stationary joint policy for $\mathcal L_\Sigma(\cdot,\lambda_k)$. We execute that profile for episode $k$, record $\widehat C_k=1-B_H^{(k)}/K$, and update
\begin{equation}\label{eq:projected-dual}
\lambda_{k+1}=[\lambda_k+\eta(\widehat C_k-\varepsilon)]_+.
\end{equation}
The price sequence selects $\boldsymbol\pi^0,\boldsymbol\pi^1,\ldots$; these stationary components constitute the nonstationary $\boldsymbol\Pi_M$ in~\eqref{eq:epoch-values}. The update uses the \emph{global terminal} cost and one shared price, even though agents have different objectives. The multiplier update is part of the policy-generating procedure: the results evaluate the sequence it produces, so feasibility does not require extracting a single limiting policy. 
Different unconstrained solvers can be used when their solution errors satisfy the stated bounds. Our analysis extends the time-average constructions of~\citet{das2025cmg,das2025lagrangian} to discounted reset episodes and terminal costs. Appendix~\ref{app:source-results} records the precise average-reward assumptions and guarantees from the two source papers. 

\begin{algorithm}[t]
\caption{Lagrangian policy-sequence framework}\label{alg:primal-dual}
\small
\begin{algorithmic}[1]
\Require $\varepsilon\in[0,1]$, $\eta>0$, $\lambda_0\geq0$, $\delta\geq0$, $M,H\in\mathbb N$, $\mathcal P$;
mode $m\in\{\mathrm{CMG},\mathrm{CM\text{-}MDP}\}$
\For{$k=0,\ldots,M-1$}
 \If{$m=\mathrm{CMG}$}
  \State $\boldsymbol\pi^k\gets\mathcal O_{\rm NE}(\lambda_k,\delta)$ satisfying, for every $i$,
  \Statex \hspace{\algorithmicindent}$\displaystyle
  \sup_{z_i\in\mathcal P_i}\mathcal L_i((z_i,\boldsymbol\pi_{-i}^k),\lambda_k)
  -\mathcal L_i(\boldsymbol\pi^k,\lambda_k)\leq\delta$
 \Else
  \State $\boldsymbol\pi^k\gets\mathcal O_{\rm team}(\lambda_k,\delta)$ satisfying
  \Statex \hspace{\algorithmicindent}$\displaystyle
  \sup_{\boldsymbol\pi\in\mathcal P}\mathcal L_\Sigma(\boldsymbol\pi,\lambda_k)
  -\mathcal L_\Sigma(\boldsymbol\pi^k,\lambda_k)\leq\delta$
 \EndIf
 \State $B_0^{(k)}\gets K$
 \For{$t=0,\ldots,H-1$}
  \State $e_{i,t}\sim\pi_i^k(\cdot\mid o_{i,t})$ for all $i$;
  $(h_{1:n,t},B_{t+1}^{(k)})\gets\operatorname{Fishery}(B_t^{(k)},e_{1:n,t})$ via~\eqref{eq:stock}
 \EndFor
 \State $\widehat C_k\gets1-B_H^{(k)}/K$
 \State $\lambda_{k+1}\gets[\lambda_k+\eta(\widehat C_k-\varepsilon)]_+$
\EndFor
\State \Return $\boldsymbol\Pi_M=(\boldsymbol\pi^0,\ldots,\boldsymbol\pi^{M-1})$,
$\{\widehat C_k,\lambda_k\}_{k<M}$
\end{algorithmic}
\end{algorithm}

Algorithm~\ref{alg:primal-dual} repeatedly chooses a harvesting policy, runs it for an episode, and uses the resulting depletion to adjust the price. At the current price, the solver seeks an approximate Nash equilibrium for individually motivated agents or an approximately optimal policy for the team. Depletion above the budget raises the price; depletion below it lowers the price, subject to remaining nonnegative. The algorithm returns the sequence of policies used, whose combined resource use and rewards are evaluated by our results. The deterministic analysis updates prices using expected depletion $C(\boldsymbol\pi^k)$. The sampled-cost analysis instead uses episode measurements whose expected value, given the information available before the episode, equals $C(\boldsymbol\pi^k)$. This is the conditional-unbiasedness assumption used in the earlier epoch constructions~\citep{das2025cmg,das2025lagrangian}. Appendix~\ref{app:algorithms} discusses the continuing, time-average interpretation and its evaluation assumptions.

\begin{proposition}[Feasibility certificate]\label{prop:average-epoch-certificate}
For any realized costs $\widehat C_k\in[0,1]$, recursion~\eqref{eq:projected-dual} implies
\begin{equation}\label{eq:certificate}
\frac1M\sum_{k=0}^{M-1}\widehat C_k-\varepsilon
\leq\frac{\lambda_M-\lambda_0}{\eta M}.
\end{equation}
Consequently $\lambda_M=o(M)$ implies realized average-epoch feasibility. If, additionally, $\mathbb E[\widehat C_k\mid\mathcal F_k]=C(\boldsymbol\pi^k)$ and episodes have bounded costs, the same asymptotic statement holds for $M^{-1}\sum_k C(\boldsymbol\pi^k)$ almost surely, provided $\lambda_M=o(M)$ almost surely.
\end{proposition}
\noindent The full proof is in Appendix~\ref{app:proof-feasibility}.

Under Proposition~\ref{prop:average-epoch-certificate}, sustained budget excess must accumulate in the multiplier. Sublinear price growth then implies average feasibility even when the selected policies keep changing. This reward-independent argument carries the policy-sequence mechanism beyond time-average rewards to our discounted-episode objective and terminal-cost requirement. Unconstrained solvers select policies with different rewards and costs, while controlling multiplier growth ensures feasibility on average. The bound averages cost minus budget, allowing epochs below the budget to compensate for epochs above it; it gives neither a final-policy nor a per-episode guarantee. An upper cap can hide continuing excess and invalidate the inference. The next result supplies cooperative multiplier control. Strategic performance requires its own certificate.

\begin{proposition}[Cooperative feasibility and reward]\label{prop:coop-bound}
Suppose $C(\pi^\dagger)\leq\varepsilon-\sigma$ for some joint policy and $\sigma>0$, $J_\Sigma^\gamma$ has range at most $R$, and the oracle is $\delta$-optimal for the exact team Lagrangian at each price. Use the expected-cost update $\lambda_{k+1}=[\lambda_k+\eta(C(\pi^k)-\varepsilon)]_+$. Set $G=\max\{\varepsilon,1-\varepsilon\}$ and $B=\max\{\lambda_0,(R+\delta)/\sigma+\eta G\}$. Then, for every $M\geq1$,
\begin{align}\label{eq:coop-bound}
 0\leq\lambda_k\leq B,
 \bar C_M\leq\varepsilon+\frac{B-\lambda_0}{\eta M},\\
 \bar J_{\Sigma,M}\geq J_{\Sigma,M}^{\star}-\delta-\frac{\eta G^2}{2}
 -\frac{\lambda_0^2}{2\eta M}.\label{eq:coop-reward-main}
\end{align}
\end{proposition}
\noindent The full proof is in Appendix~\ref{app:proof-cooperative}.

Proposition~\ref{prop:coop-bound} is the framework's cooperative performance guarantee, built on Proposition~\ref{prop:average-epoch-certificate}'s feasibility mechanism. In the cooperative commons, unconstrained team optimization produces a policy sequence meeting the average terminal-stock requirement up to the stated tolerance, with discounted harvest reward within the stated shortfall of the best average-feasible sequence. Abstention establishes that the budget is attainable; the reward bound measures how effectively the framework uses the resource. Neither guarantee requires convergence to one harvesting policy.

Proposition~\ref{prop:coop-bound} supplies the price control left open by Proposition~\ref{prop:average-epoch-certificate}. Joint abstention gives $C(\pi^\dagger)=0$, hence strict slack $\sigma=\varepsilon>0$. At sufficiently high prices, approximate team optimization must select a policy meeting the budget, keeping prices bounded. The feasibility tolerance is $O((\eta M)^{-1})$; the reward shortfall contains the solver error $\delta$, step-size term $\eta G^2/2$, and initialization term. With $\lambda_0=0$ and $\eta=M^{-1/2}$, both finite-horizon terms are $O(M^{-1/2})$. Constants worsen as $\sigma$ shrinks, and abstention provides no strict slack at $\varepsilon=0$. This builds on averaging in dual subgradient methods~\citep{larsson1999ergodic,nedic2009primal} and policy mixtures in constrained RL~\citep{le2019batch,miryoosefi2019convex}, with performance compared against feasible policy sequences.

Theorem~\ref{thm:sampled} in Appendix~\ref{app:proof} removes the need to observe exact expected costs in the cooperative update. Under conditional unbiasedness and the same cooperative oracle and strict-feasibility assumptions, $\lambda_M=O(\log M)$ almost surely, and both sampled and conditional-expected average depletion are asymptotically feasible. For fixed $\eta$ and confidence, its finite-sample cost bound includes an $O(M^{-1/2})$ sampling term. Its reward guarantee holds in expectation, with the same oracle and step-size errors.

Theorem~\ref{thm:nash} provides the framework's strategic performance guarantee. It bounds an individual harvester's gain from changing policies under Definition~\ref{def:epoch-nash}, a question distinct from resource protection or team reward. Its full proof is in Appendix~\ref{app:proof-nash}.

\begin{theorem}[Finite-sequence constrained Nash certificate]\label{thm:nash}
Suppose every $\pi^k$ is a $\delta$-Nash profile of the game with payoffs $\mathcal L_i(\cdot,\lambda_k)$, and use the expected-cost projected update with constant $\eta>0$. Set $G=\max\{\varepsilon,1-\varepsilon\}$. Put
\[
 \alpha_M=\left[\frac{\lambda_M-\lambda_0}{\eta M}\right]_+,
 \qquad
 \Delta_M=\inf_{\ell\geq0}\frac1M\sum_{k<M}|\lambda_k-\ell|.
\]
Then $\boldsymbol\Pi_M$ satisfies Definition~\ref{def:epoch-nash} with feasibility tolerance $\alpha_M$ and deviation tolerance
\begin{equation}\label{eq:nash-certificate}
 \delta_M=\delta+\frac{\eta G^2}{2}
             +\frac{\lambda_0^2}{2\eta M}+G\Delta_M.
\end{equation}
This is an a posteriori certificate. It does not assert that $\alpha_M$ or $\Delta_M$ vanishes.
\end{theorem}

For individually motivated harvesters, Theorem~\ref{thm:nash} connects approximate equilibria of unconstrained Lagrangian games to a constrained policy sequence: $\alpha_M$ bounds average budget excess, and $\delta_M$ bounds each agent's gain from any average-feasible deviation sequence against the opponents' fixed sequence. Unlike the team reward comparison in Proposition~\ref{prop:coop-bound}, this tests whether reallocating one's own depletion budget across epochs can improve individual harvest. The additional $G\Delta_M$ term accounts for spending more of that budget at high prices and compensating at low prices. For example, excess costs $(-a,+a)$ at prices $(0,L)$ have zero mean but price-weighted mean $La/2$. Thus bounded prices suffice for asymptotic feasibility, while persistent price oscillations can keep the strategic error bound positive. Price convergence to a finite limit gives $\Delta_M\to0$; at fixed $\eta$, the remaining bound is $\delta+\eta G^2/2$. The cooperative and strategic results thus address both collective reward and individual incentives under the resource budget.

The experiments use independent PPO for individual rewards and MAPPO for the team reward~\citep{schulman2017proximal,dewitt2020ippo,yu2022surprising}. After each episode, a shared controller adjusts the depletion price using both the current budget excess and an accumulated record of past excesses. This feedback guides subsequent policy updates. Several budget regimes exhibit depletion below or near their targets with approximate PPO updates, dense reward shaping, and capped proportional--integral (PI) feedback, showing useful resource regulation beyond the precise implementation analyzed in the theory. Algorithms~\ref{alg:ippo} and~\ref{alg:mappo} in Appendix~\ref{app:algorithms} describe the learners; Appendix~\ref{app:protocol} gives the controller and training settings.

\section{Experiments}\label{sec:experiments}

We evaluate constrained IPPO and MAPPO at depletion budgets $\varepsilon\in\{0.01,0.05,0.1,0.3,0.4,0.6\}$, with \editorialflag{three seeds per method--budget pair.} Terminal depletion measures the resource requirement, harvest return measures its reward consequences, and the price records feedback used to regulate depletion. Actors and prices change across training epochs, generating a sequence of harvesting policies.

The three experimental figures show depletion, unpenalized return, and the shared price over the final 2,000 epochs. Each curve reports a seed mean and its envelope spans the seed minimum--maximum. These training measurements describe late-window tracking and adaptation, not held-out evaluation of a final policy or the full-sequence quantity $\bar C_M$ in our solution concept.

Each run comprises 20,000 training epochs. Protocol details are in Appendix~\ref{app:protocol}.

\subsection{Constraint tracking is budget- and window-dependent}\label{sec:constrained-solution}

Figure~\ref{fig:convergence} examines whether the depletion budget provides practical control over resource use. Smaller budgets demand greater terminal-stock preservation, and the learned depletion trajectories reflect these different requirements in both the individual-reward and team-reward settings. The smallest budget produces strong stock preservation; larger budgets permit greater depletion, with intermediate budgets showing more variation across seeds and epochs. The reference lines mark the allowed depletion $\varepsilon$. Curves below a reference line meet that budget in the displayed measurement, while crossings identify excess depletion. The seed envelopes show that a mean near the target can coexist with individual seeds above it.

\begin{figure}[t]
\centering\includegraphics[width=\linewidth]{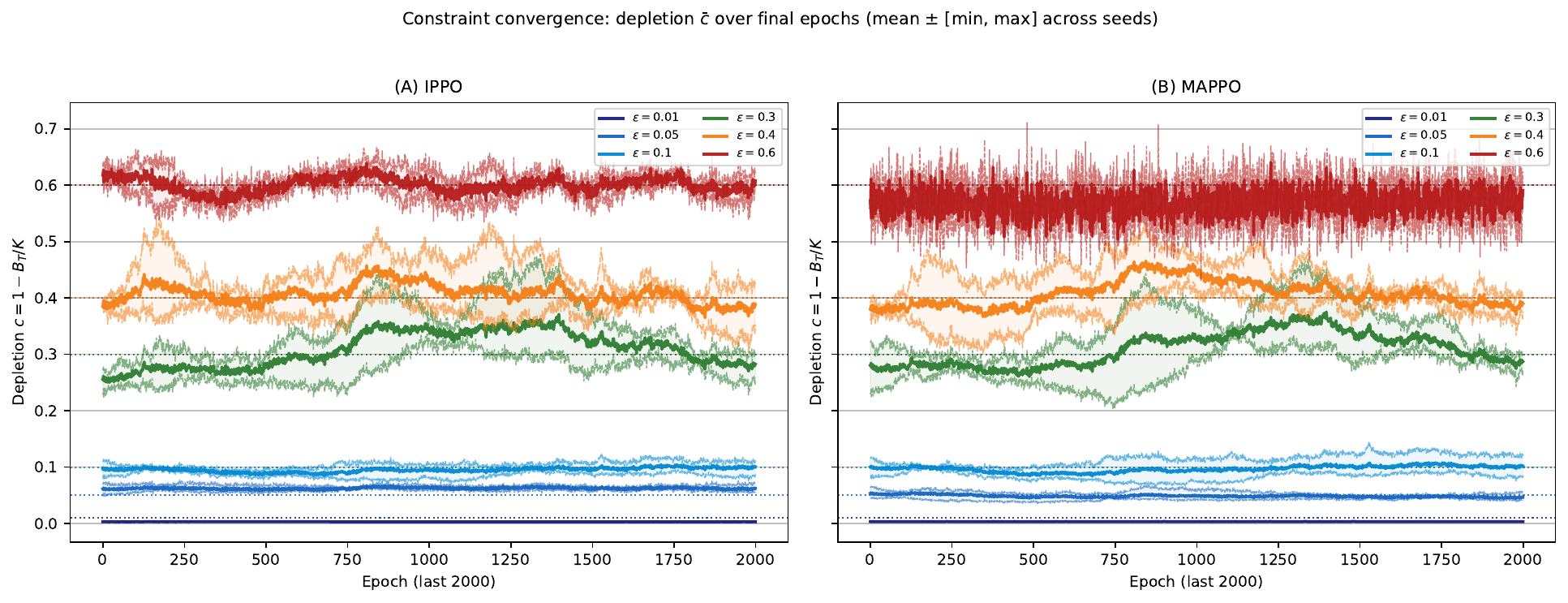}
\caption{Terminal depletion: horizontal axis is the epoch offset within the final 2,000 epochs; vertical axis is $c=1-B_T/K$ ($T=H$). Panels (A)/(B) show IPPO/MAPPO. Colors denote budgets $\varepsilon$; dotted horizontal lines mark their targets. Curves show seed means. Envelopes span the seed minimum--maximum. See Appendix~\ref{app:plot-reading} for notation.}
\label{fig:convergence}
\end{figure}

\subsection{The reward--sustainability tradeoff}\label{sec:reward}

The resource constraint deliberately changes the optimization problem. Figure~\ref{fig:reward} shows the harvest return associated with different depletion budgets over the same displayed late-training horizon. The tightest budget is accompanied by low return; larger budgets generally show higher returns, with little separation between MAPPO at budgets $.4$ and $.6$. The budget specifies how much terminal stock can be sacrificed in pursuit of harvest.

Harvest return and depletion should be read together: each budget defines a different resource requirement. The theoretical reward benchmark compares sequences satisfying the same budget; the experimental curves show the harvest associated with each chosen budget.

\begin{figure}[t]
\centering\includegraphics[width=\linewidth]{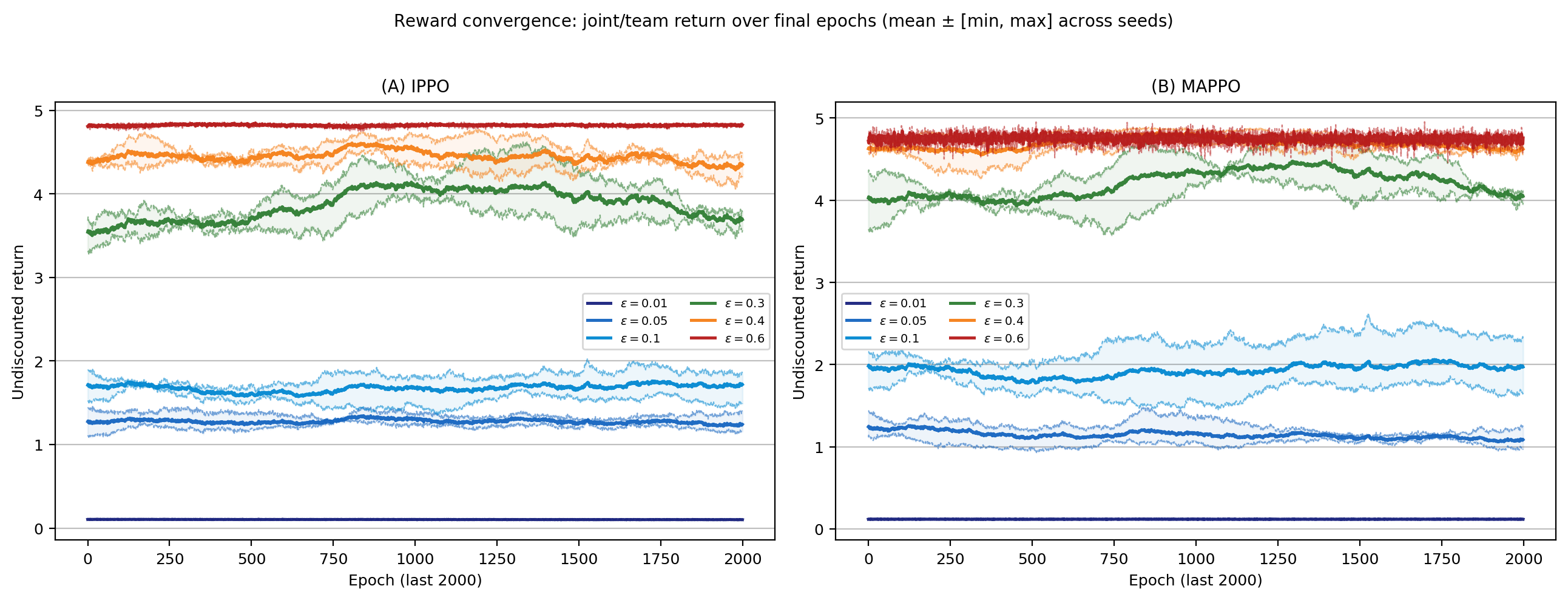}
\caption{Harvest return: horizontal axis is the epoch offset within the final 2,000 epochs; vertical axis is undiscounted joint/team return (the trainer logs total harvest divided by $K$). Panels (A)/(B) show IPPO/MAPPO; colors denote $\varepsilon$. Curves show seed means. Envelopes span their minimum--maximum. This is not the discounted objective $J_\Sigma^\gamma$; see Appendix~\ref{app:plot-reading}.}
\label{fig:reward}
\end{figure}

In Fig.~\ref{fig:lambda-main}, tight budgets sustain positive prices, while MAPPO's price is near zero at $\varepsilon=0.6$, consistent with little sustained price pressure being needed in that displayed loose-budget regime. Price magnitude also depends on reward scale, penalty scale, update interval, and PI gains, so it cannot directly measure reward foregone.

\begin{figure}[t]
\centering\includegraphics[width=\linewidth]{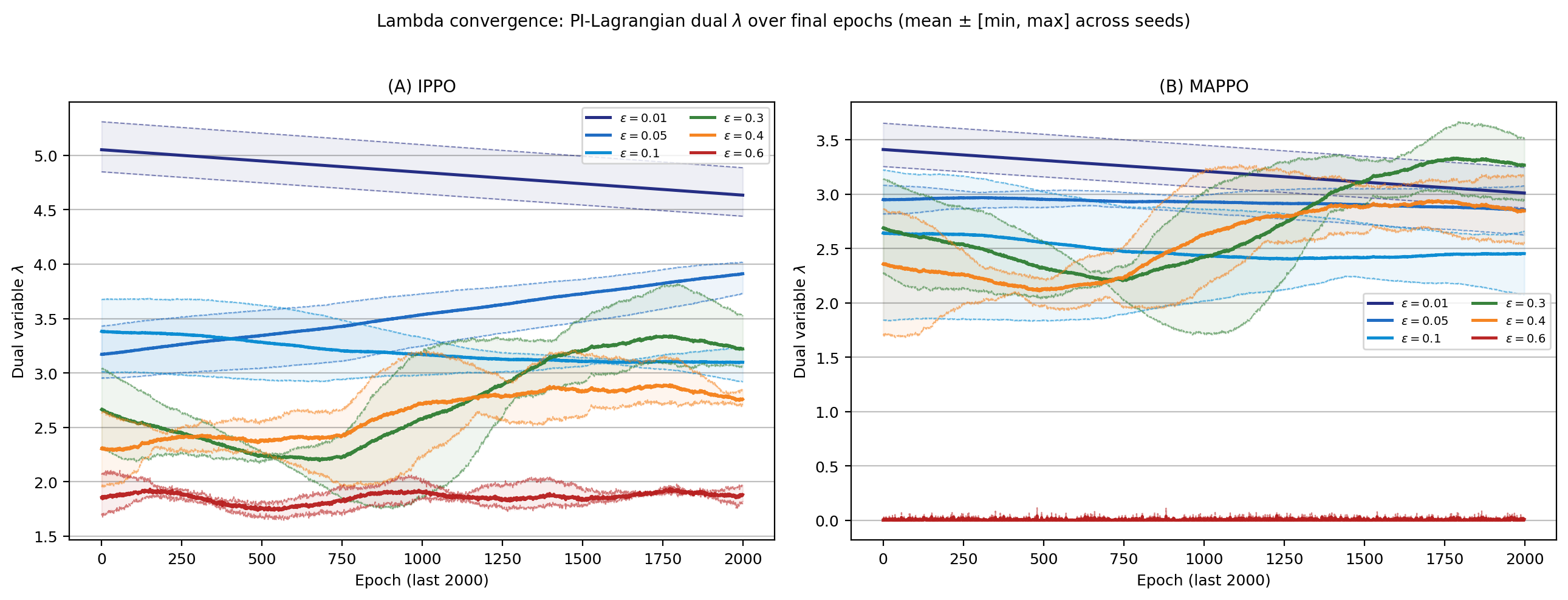}
\caption{Ecological price: horizontal axis is the epoch offset within the final 2,000 epochs; vertical axis is the shared PI multiplier $\lambda$. Panels (A)/(B) show IPPO/MAPPO; colors denote $\varepsilon$. Curves show seed means and envelopes span seed ranges. This $\lambda$ is distinct from $\lambda_{\rm GAE}$.}\label{fig:lambda-main}
\end{figure}

The price trajectory records feedback between budget excess and learning. Positive excess depletion increases the integral state until its cap and contributes a positive proportional term; slack has the opposite effect, subject to projection. The total price need not increase after every violating epoch, since the proportional term also changes from the previous epoch. Actors then adapt to the changed penalty, which can alter subsequent depletion. This delayed interaction can generate oscillation even when the budget is fixed~\citep{stooke2020responsive}.

Figure~\ref{fig:lambda-main} shows continued adaptation within the displayed window. IPPO retains its largest price at the tightest budget. MAPPO at budget $.3$ shows a price that first falls and then rises, while its price at $.6$ stays near zero. Learned prices need not decrease monotonically with the budget at each epoch because each budget induces a different policy and controller history. Persistent price variation is compatible with changing policies; under the game certificate it also enters the deviation bound through price dispersion.

\subsection{Empirical scope and implementation limits}\label{sec:empirical-limits}

The evaluation examines how the nonstationary Lagrangian framework regulates resource use in the constrained commons. Establishing how frequently unconstrained MARL leads to depletion would require a separate statistical study across learning algorithms, environments, and systematically varied---including adversarially selected---parameters. Here we examine budget tracking, harvest reward, and price adaptation with practical learners that depart from the exact optimization and multiplier updates assumed by the theory.

The evaluation uses stochastic policies in one two-agent environment with deterministic stock dynamics. PPO performs updates on shaped rewards, and its dense one-sided reward shaping and capped PI feedback differ from the exact terminal-cost Lagrangian and uncapped projected update analyzed in the theory. Several displayed regimes exhibit near-target depletion with this implementation. Excursions rule out per-epoch feasibility in the displayed window; they neither establish nor refute full-sequence average feasibility. The experiments do not measure oracle error, feasible-deviation gains, or cooperative optimality gaps.

Joint abstention verifies cooperative strict feasibility when included in the policy class. The uniform unilateral Slater condition in the prior CMG result fails in this fishery: one agent cannot offset maximal extraction by its opponent (Appendix~\ref{app:environment}). This explains why observed stock regulation is assessed separately from a Nash claim. 

The ecological measurement is terminal depletion. It does not control within-episode stock minima or ensure safety in every episode. Reset-episode averaging is appropriate when resource requirements may be balanced across episodes; irreversible-collapse applications instead require control of the relevant trajectory event, potentially with a safety filter~\citep{wabersich2021filter}.

\section{Conclusion}\label{sec:conclusion}

The tragedy of the commons motivates a learning problem in which resource preservation is an explicit requirement. We formulate this requirement through a depletion budget and develop an average-epoch solution concept for nonstationary policy sequences. Our Lagrangian framework connects solutions of unconstrained games and team problems to feasibility, cooperative reward, and constrained Nash guarantees under stated assumptions. The fishery experiments illustrate how the budget shapes terminal stock, harvest return, and price adaptation, through the generated policy sequences. The resulting approach gives resource protection a specified quantitative target and provides a mathematical basis for evaluating the policies generated as agents learn to use the commons.

\paragraph{Reproducibility statement.} Appendix~\ref{app:protocol} specifies the model, \editorialflag{training protocol}, and plot conventions; Appendix~\ref{app:diagnostics} discusses the interpretation of the experimental curves.

\paragraph{AI-use statement.} An AI assistant was used to help develop mathematical arguments, interpret experimental results, identify references, and edit the manuscript and LaTeX. The authors are responsible for all scientific statements.

\bibliographystyle{plainnat}
\bibliography{references-main-4}
\clearpage
\appendix
\setcounter{figure}{0}\setcounter{table}{0}
\renewcommand{\thefigure}{A\arabic{figure}}
\renewcommand{\thetable}{A\arabic{table}}
\setcounter{algorithm}{0}
\renewcommand{\thealgorithm}{A\arabic{algorithm}}
\providecommand{\theHalgorithm}{}
\renewcommand{\theHalgorithm}{appendix.\arabic{algorithm}}
\renewcommand{\theHfigure}{appendix.\arabic{figure}}
\renewcommand{\theHtable}{appendix.\arabic{table}}

\section{Relation to the average-reward Lagrangian results}\label{app:source-results}

The epoch-dependent construction in this paper follows the architecture of the constrained-Markov-game result of~\citet{das2025cmg} and the cooperative CM-MDP result of~\citet{das2025lagrangian}. Both source results concern a continuously evolving, infinite-horizon system evaluated by long-run \emph{time-average} reward and constraint values. Our fishery experiment instead resets the stock between finite episodes, optimizes a discounted episodic reward, and measures a terminal cost. These differences leave the algebra of the dual update intact but change how policies are evaluated, which alternatives they are compared with, and which problems the solvers must solve.

\subsection{Original constrained Markov-game result}

In the original CMG formulation, each agent $i$ has a long-run average reward $V_i^s(\pi)$ from initial state $s$, and the common or coupled constraint values are $U_j^s(\pi)\geq b_j$. A nonstationary policy changes with physical time; an epoch-dependent policy holds a stationary profile $\pi^k$ for $T_0$ time steps before switching. The Lagrangian game at multiplier vector $\lambda_k\geq0$ gives agent $i$ payoff
\begin{equation}\label{eq:source-cmg}
\mathcal L_i^s(\pi,\lambda_k)=V_i^s(\pi)+\lambda_k^\top[U^s(\pi)-b].
\end{equation}
The oracle must return a stationary $\delta$-Nash profile of \emph{each} encountered Lagrangian game. Rollout constraint values then update the multiplier by projected dual descent. The sign in~\eqref{eq:source-cmg} reflects a lower-bound constraint; our upper-bound terminal depletion in~\eqref{eq:lagrangians} uses the equivalent negative sign.

The original theorem assumes (i) a stationary approximate-Nash oracle; (ii) a \emph{uniform unilateral Slater condition}: against any stationary policies of the other players, each player can unilaterally produce strictly feasible constraint values; (iii) bounded reward and constraint functions; and (iv) unbiased epoch rollouts for the associated average-reward values. Under these assumptions it states almost-sure long-run feasibility and an approximate nonstationary Nash guarantee, with a gap of order $\eta B^2/2+\delta$ in its notation~\citep{das2025cmg}. The proof uses player-specific generalized dual functions to control the multiplier sequence, because a general-sum game has no single team dual objective.

More explicitly, the average-reward CMG oracle returns $\pi^k$ such that, for every agent $i$ and stationary unilateral alternative $z_i$,
\begin{equation}\label{eq:source-game-oracle}
\mathcal L_i^s(\pi^k,\lambda_k)+\delta\geq
\mathcal L_i^s((z_i,\pi_{-i}^k),\lambda_k).
\end{equation}
The original analysis controls multiplier growth through the family of unilateral generalized duals
\begin{equation}\label{eq:source-generalized-dual}
d_i^s(\lambda,\pi_{-i})=
\sup_{z_i}\mathcal L_i^s((z_i,\pi_{-i}),\lambda).
\end{equation}
The arguments of~\eqref{eq:source-generalized-dual} include the opponents' current policies, so even the dual-control step is not the minimization of one scalar team function. \editorialflag{The game formulation evaluates individual incentives, whereas the cooperative formulation evaluates a joint objective.}

Two restrictions matter here. First, the uniform unilateral Slater condition need not hold for a shared renewable stock: if one harvester extracts maximally, the other cannot restore biomass by abstaining. Second, a result for average \emph{time-step} reward and feasible deviations cannot, without further argument, establish Definition~\ref{def:epoch-nash}, whose deviations may reallocate a finite constraint budget among reset episodes. The solution concept in Definition~\ref{def:epoch-nash} is an episodic analogue motivated by the original work, not a restatement of its theorem.

\subsection{Original safe cooperative-MARL result}

The cooperative CM-MDP of~\citet{das2025lagrangian} has one joint average-reward value $V^s(\pi)$ and system constraints $U_j^s(\pi)\geq b_j$. At multiplier $\lambda_k$, a learning oracle returns a stationary $\delta$-optimal policy for the relaxed joint MDP with objective $V^s(\pi)+\lambda_k^\top[U^s(\pi)-b]$. The agents execute this stationary policy over an epoch, update the dual from sampled constraint values, and concatenate the sequence to obtain an epoch-dependent nonstationary policy. The theorem assumes a joint strictly feasible stationary policy, bounded values, an approximately optimal oracle, and unbiased rollouts. Under those assumptions, the executed infinite-horizon sequence is almost surely feasible and its joint average reward is within the reported $\eta B^2/2+\delta$ tolerance of the paper's stationary feasible benchmark. The benchmark is not a Nash condition: cooperative agents maximize one joint value.

The cooperative source oracle condition is $\mathcal L^s(\pi^k,\lambda_k)+\delta\geq\sup_{\pi\in\mathcal P}\mathcal L^s(\pi,\lambda_k)$, where $\mathcal L^s$ is the joint average-reward Lagrangian. Its performance is compared with the best \emph{stationary feasible} joint policy, even though the returned policy is epoch-dependent and nonstationary. The cooperative paper's proof combines multiplier tightness, feasibility, and a reward lower bound; none of these three conclusions alone implies the other two. For reset episodes, Proposition~\ref{prop:coop-bound} proves feasibility and a reward bound against every feasible mixture; Theorem~\ref{thm:sampled} treats unbiased sampled costs. These are separate results proved here. We do not establish comparable bounds on solution error or optimality for PPO.

Compared with the CMG case, joint strict feasibility is plausible in the fishery because both agents can refrain from extraction. The implemented policy update uses a shaped one-sided biomass-loss signal; the dual uses terminal depletion; and its proportional--integral recursion has a capped state. \editorialflag{The shaping discrepancy contributes to the oracle error for the terminal-cost objective, as quantified in Appendix~\ref{app:mismatch}.}

\section{Complete proofs and the average-epoch equilibrium question}\label{app:proof}\label{app:theory}

Throughout this appendix, $0<\varepsilon<1$, $g(\pi)=C(\pi)-\varepsilon$, and $G=\max\{\varepsilon,1-\varepsilon\}$. Since costs lie in $[0,1]$, both $g(\pi)$ and $\widehat g_k:=\widehat C_k-\varepsilon$ belong to $[-\varepsilon,1-\varepsilon]$, and their absolute values are at most $G$. We write $g_k=g(\pi^k)$ and $\bar C_M=M^{-1}\sum_{k<M}C(\pi^k)$. The initialization $\lambda_0\geq0$ is finite, and the step size $\eta>0$ is constant throughout each run. Projection is onto $[0,\infty)$, with no upper cap.

For stochastic statements, work on a probability space with an increasing filtration $(\mathcal F_k)_{k\geq0}$. The policy $\pi^k$ and price $\lambda_k$ are $\mathcal F_k$-measurable, while the evaluation cost $\widehat C_k$ is $\mathcal F_{k+1}$-measurable. Thus $\mathcal F_k$ contains the preceding rollout outcomes and any randomization used to select the current policy, but not its forthcoming evaluation outcome. Conditional unbiasedness means $\mathbb E[\widehat C_k\mid\mathcal F_k]=C(\pi^k)$; it does not require independence across epochs. Any almost-sure conclusion that assumes $\lambda_M/M\to0$ requires that growth condition to hold almost surely. Countably many almost-sure oracle conditions, one for each epoch, can be imposed on a single probability-one event.

All policies and comparators belong to the specified policy class $\mathcal P$, and the cooperative strictly feasible policy must belong to that same class. Bounded rewards ensure that the value comparisons below are finite. The constant $R$ bounds the range of the relevant reward value, rather than the magnitude of a single sampled reward. The oracle tolerance $\delta\geq0$ is a uniform error in the value of the exact Lagrangian; it is not a PPO clipping parameter, gradient norm, or empirical training loss. We retain explicit intermediate inequalities to distinguish the algebraic certificate, multiplier control, reward comparison, and strategic deviation test.

\subsection{Full proof of Proposition~\ref{prop:average-epoch-certificate}}\label{app:proof-feasibility}
\begin{proof}
We first fix an arbitrary realization of the costs. For every real number $x$, $[x]_+=\max\{x,0\}\geq x$. Applying this inequality to the multiplier update gives
\[
 \lambda_{k+1}=[\lambda_k+\eta\widehat g_k]_+
 \geq\lambda_k+\eta\widehat g_k,
 \qquad
 \eta(\widehat C_k-\varepsilon)\leq\lambda_{k+1}-\lambda_k.
\]
Summing over $k=0,\ldots,M-1$ cancels every intermediate multiplier:
\[
 \eta\sum_{k=0}^{M-1}(\widehat C_k-\varepsilon)
 \leq\sum_{k=0}^{M-1}(\lambda_{k+1}-\lambda_k)
 =\lambda_M-\lambda_0.
\]
Division by $\eta M>0$ proves~\eqref{eq:certificate}. This calculation is pathwise and uses no reward or oracle assumption. If $\lambda_M/M\to0$ on the chosen realization, then $(\lambda_M-\lambda_0)/(\eta M)\to0$, because $\lambda_0$ and $\eta$ are fixed. Taking a limit superior therefore gives realized average feasibility on that realization. In particular, an almost-sure sublinear-growth assumption yields an almost-sure realized-cost conclusion.

We now relate the observed costs to the values of the selected policies. Set
\[
 D_k=\widehat C_k-C(\pi^k),\qquad S_M=\sum_{k=0}^{M-1}D_k,\qquad S_0=0.
\]
Conditional unbiasedness gives $\mathbb E[D_k\mid\mathcal F_k]=0$. Moreover, $D_k\in[-1,1]$, since both costs lie in $[0,1]$. The variables $D_k$ are thus bounded martingale differences: they need not be independent, but their conditional means vanish before each evaluation. The conditional form of Hoeffding's lemma (Lemma~\ref{lem:conditional-hoeffding}, Appendix~\ref{app:probability-tools}), applied to a centered variable in an interval of length two, gives
\[
 \mathbb E[e^{tD_k}\mid\mathcal F_k]\leq e^{t^2/2},\qquad t\in\mathbb R.
\]
The bound is conservative but sufficient. Since $S_{M-1}$ is $\mathcal F_{M-1}$-measurable, the tower property (Lemma~\ref{lem:conditional-expectation}, Appendix~\ref{app:probability-tools}) yields
\begin{align*}
 \mathbb E e^{tS_M}
 &=\mathbb E\!\left[e^{tS_{M-1}}
       \mathbb E[e^{tD_{M-1}}\mid\mathcal F_{M-1}]\right]\\
 &\leq e^{t^2/2}\mathbb E e^{tS_{M-1}}
 \leq\cdots\leq e^{Mt^2/2}.
\end{align*}
For $a>0$ and $t>0$, Markov's inequality (Lemma~\ref{lem:markov}, Appendix~\ref{app:probability-tools}) consequently gives
\[
 \Pr(S_M\geq Ma)\leq e^{-tMa}\mathbb E e^{tS_M}
 \leq\exp\{-tMa+Mt^2/2\}.
\]
The exponent is minimized at $t=a$, so this probability is at most $e^{-Ma^2/2}$. Applying the same calculation to $-S_M$ and adding the two tail probabilities proves
\begin{equation}\label{eq:martingale-tail}
 \Pr\!\left(\left|\sum_{k<M}D_k\right|\geq Ma\right)
 \leq2e^{-Ma^2/2},\qquad a>0.
\end{equation}
For each fixed positive integer $j$, the sum over $M$ of the right-hand side with $a=1/j$ is finite. The first Borel--Cantelli lemma (Lemma~\ref{lem:borel-cantelli}, Appendix~\ref{app:probability-tools}) implies that $|S_M|/M\geq1/j$ occurs only finitely often, almost surely. Taking the intersection of these probability-one events, using Lemma~\ref{lem:union-bound} in Appendix~\ref{app:probability-tools}, over the countable set $j=1,2,\ldots$ proves $S_M/M\to0$ almost surely. No independence of the tail events is needed.

By the definition of $S_M$,
\[
 \bar C_M-\varepsilon
 =\frac1M\sum_{k<M}\widehat C_k-\varepsilon-\frac{S_M}{M}.
\]
On the event where both $\lambda_M/M\to0$ and $S_M/M\to0$, the realized-cost certificate therefore implies $\limsup_M\bar C_M\leq\varepsilon$. This proves the conditional-expected assertion under almost-sure sublinear multiplier growth.

For a fixed horizon $M$ and $\beta\in(0,1)$, the lower-tail calculation above gives
\[
 \Pr\!\left(-\frac{S_M}{M}>\sqrt{\frac{2\log(1/\beta)}{M}}\right)\leq\beta.
\]
Substituting this bound into the preceding identity and using the pathwise certificate proves, with probability at least $1-\beta$,
\begin{equation}\label{eq:finite-sample-feas}
 \bar C_M-\varepsilon\leq
 \frac{\lambda_M-\lambda_0}{\eta M}
 +\sqrt{\frac{2\log(1/\beta)}{M}}.
\end{equation}
Unlike the asymptotic conclusion, this fixed-horizon inequality does not assume sublinear multiplier growth.

Finally, suppose the evaluations have conditional bias
$\mathbb E[\widehat C_k\mid\mathcal F_k]=C(\pi^k)+b_k$ with $|b_k|\leq b_k^{\max}$. Center instead at the actual conditional mean, setting $\widetilde D_k=\widehat C_k-\mathbb E[\widehat C_k\mid\mathcal F_k]$. This variable is still conditionally centered and in $[-1,1]$, since a conditional mean of a $[0,1]$-valued variable lies in $[0,1]$. Now
\[
 \bar C_M=\frac1M\sum_{k<M}\widehat C_k
 -\frac1M\sum_{k<M}\widetilde D_k-\frac1M\sum_{k<M}b_k.
\]
The same concentration argument applies to $\widetilde D_k$, while $-\sum_k b_k\leq\sum_k b_k^{\max}$. Thus the right-hand side of~\eqref{eq:finite-sample-feas} acquires the additional term $M^{-1}\sum_{k<M}b_k^{\max}$. Vanishing average bias preserves the asymptotic conclusion; a persistent bias need not.
\end{proof}

The proposition is an accounting identity plus a sampling argument. It does not show that the price grows sublinearly. For example, $\widehat C_k=1$ gives $\lambda_M=\lambda_0+\eta M(1-\varepsilon)$, for which the bound is exact and infeasibility persists. This is why Proposition~\ref{prop:coop-bound} is needed. The full induction argument and its reward and sampling extensions follow here.

\subsection{A dual-energy lemma}
\begin{lemma}[Projected-dual energy]\label{lem:energy}
For $\lambda_{k+1}=[\lambda_k+\eta v_k]_+$, $\lambda_k\geq0$, and $|v_k|\leq G$,
\begin{equation}\label{eq:energy}
 -\frac1M\sum_{k<M}\lambda_kv_k
 \leq \frac{\lambda_0^2-\lambda_M^2}{2\eta M}
       +\frac{\eta}{2M}\sum_{k<M}v_k^2
 \leq \frac{\lambda_0^2}{2\eta M}+\frac{\eta G^2}{2}.
\end{equation}
\end{lemma}
\begin{proof}
The squared multiplier provides a telescoping quantity that controls the price-weighted violations. For every real $x$, $[x]_+^2\leq x^2$: equality holds for $x\geq0$, and the left-hand side is zero for $x<0$. Therefore
\[
 \lambda_{k+1}^2\leq(\lambda_k+\eta v_k)^2
 =\lambda_k^2+2\eta\lambda_kv_k+\eta^2v_k^2.
\]
Moving the product term to the left and dividing by $2\eta>0$ gives
\[
 -\lambda_kv_k\leq
 \frac{\lambda_k^2-\lambda_{k+1}^2}{2\eta}
 +\frac{\eta v_k^2}{2}.
\]
We now sum over $k<M$ and divide by $M$. The differences of consecutive squared multipliers telescope, yielding
\[
 -\frac1M\sum_{k<M}\lambda_kv_k
 \leq\frac{\lambda_0^2-\lambda_M^2}{2\eta M}
 +\frac{\eta}{2M}\sum_{k<M}v_k^2.
\]
Since $\lambda_M^2\geq0$, removing its negative contribution can only increase the right-hand side. Since $|v_k|\leq G$, the remaining average of squares is at most $G^2$. These observations give the second inequality in~\eqref{eq:energy}. The entire argument is pathwise, so the lemma applies to expected violations or sampled violations, provided they are the quantities actually used in the projected update.
\end{proof}

\subsection{Full proof of Proposition~\ref{prop:coop-bound}}\label{app:proof-cooperative}
Let $J=J_\Sigma^\gamma$. A finite mixture $\mu$ selects one entire joint stationary policy at the beginning of an episode and executes it throughout that episode. Define $J(\mu)=\mathbb E_{\pi\sim\mu}J(\pi)$ and $C(\mu)=\mathbb E_{\pi\sim\mu}C(\pi)$, and let
\begin{equation}\label{eq:mixture-opt}
 J_{\rm mix}^{\star}=\sup_{\mu:\,C(\mu)\leq\varepsilon}J(\mu),
\end{equation}
where the supremum is over finite mixtures of policies in $\mathcal P$. Every length-$M$ sequence induces an equal-weight mixture, so $J_{\Sigma,M}^{\star}\leq J_{\rm mix}^{\star}$. No existence of a maximizing mixture is needed for the bounds below.

\begin{proof}
We first use strict feasibility to control the multipliers. Let $\pi^\dagger\in\mathcal P$ satisfy $g(\pi^\dagger)\leq-\sigma$. The $\delta$-optimality assumption implies comparison with every policy in $\mathcal P$, in particular
\[
 J(\pi^k)-\lambda_kg_k
 \geq J(\pi^\dagger)-\lambda_kg(\pi^\dagger)-\delta
 \geq J(\pi^\dagger)+\lambda_k\sigma-\delta.
\]
The second inequality uses both $\lambda_k\geq0$ and the strictly negative violation of $\pi^\dagger$. Rearranging and using $J(\pi^k)-J(\pi^\dagger)\leq R$ gives
\begin{equation}\label{eq:oracle-drift}
 \lambda_kg_k
 \leq J(\pi^k)-J(\pi^\dagger)+\delta-\lambda_k\sigma
 \leq R+\delta-\lambda_k\sigma.
\end{equation}
Thus a sufficiently high price makes the strictly feasible comparator competitive with every possible reward advantage of a violating policy.

To make this observation quantitative, put $L=(R+\delta)/\sigma$ and $B=\max\{\lambda_0,L+\eta G\}$. Suppose first that $\lambda_k\geq L$ and $\lambda_k>0$. The right-hand side of~\eqref{eq:oracle-drift} is nonpositive; division by $\lambda_k$ gives $g_k\leq0$. Hence both zero and $\lambda_k+\eta g_k$ are at most $\lambda_k$, and
\[
 \lambda_{k+1}=\max\{0,\lambda_k+\eta g_k\}\leq\lambda_k.
\]
If instead $\lambda_k<L$, then $g_k\leq G$ and $\lambda_k+\eta G\geq0$, so
\[
 \lambda_{k+1}\leq\lambda_k+\eta G<L+\eta G\leq B.
\]
The only case not covered by these alternatives is $L=\lambda_k=0$, for which the same bound gives $\lambda_{k+1}\leq\eta G\leq B$. Starting with $0\leq\lambda_0\leq B$, induction now proves $0\leq\lambda_k\leq B$ for every $k$: above the threshold the multiplier cannot increase, and below it a single update cannot overshoot by more than $\eta G$.

We now invoke Proposition~\ref{prop:average-epoch-certificate} with the costs $C(\pi^k)$ actually used in this expected-cost update. Its deterministic telescoping inequality yields
\[
 \bar C_M-\varepsilon
 \leq\frac{\lambda_M-\lambda_0}{\eta M}
 \leq\frac{B-\lambda_0}{\eta M},
\]
which is the claimed feasibility bound.

For the reward comparison, fix any feasible finite mixture $\mu=\sum_{j=1}^m p_j\delta_{\pi_j}$, where $p_j\geq0$, $\sum_jp_j=1$, $\pi_j\in\mathcal P$, and $C(\mu)\leq\varepsilon$. Here $\delta_{\pi_j}$ denotes the point mass at $\pi_j$, not the oracle tolerance. At epoch $k$, the oracle inequality holds against each support policy:
\[
 J(\pi^k)-\lambda_kg_k\geq J(\pi_j)-\lambda_kg(\pi_j)-\delta.
\]
Multiply by $p_j$ and sum over $j$. Linearity of mixture values, together with $\sum_jp_j=1$, gives
\begin{align*}
 J(\pi^k)-\lambda_kg_k
 &\geq\sum_jp_jJ(\pi_j)-\lambda_k\sum_jp_jg(\pi_j)-\delta\\
 &=J(\mu)-\lambda_k[C(\mu)-\varepsilon]-\delta\\
 &\geq J(\mu)-\delta.
\end{align*}
The last inequality uses feasibility of the mixture and $\lambda_k\geq0$. It does not require each support policy to be feasible. Nor must the oracle itself optimize over mixtures: averaging its comparisons with individual policies is enough.

Rearranging and averaging over epochs gives
\[
 \bar J_{\Sigma,M}\geq J(\mu)-\delta
 +\frac1M\sum_{k<M}\lambda_kg_k.
\]
The update uses $g_k$, so Lemma~\ref{lem:energy} applies with $v_k=g_k$ and gives
\[
 \frac1M\sum_{k<M}\lambda_kg_k
 \geq-\frac{\lambda_0^2}{2\eta M}-\frac{\eta G^2}{2}.
\]
Combining these two inequalities bounds reward below by $J(\mu)$ minus the stated error, uniformly over every feasible finite mixture. Strict feasibility makes this comparator set nonempty, and bounded rewards make its supremum finite. For any $a>0$, the definition of supremum supplies a feasible mixture with reward greater than $J_{\rm mix}^{\star}-a$. Applying the uniform bound and letting $a\downarrow0$ proves
\begin{equation}\label{eq:coop-reward-full}
 \bar J_{\Sigma,M}\geq J_{\rm mix}^{\star}
 -\delta-\frac{\eta G^2}{2}-\frac{\lambda_0^2}{2\eta M}
 \geq J_{\Sigma,M}^{\star}
 -\delta-\frac{\eta G^2}{2}-\frac{\lambda_0^2}{2\eta M}.
\end{equation}
The final inequality holds because any feasible length-$M$ sequence induces a feasible mixture assigning mass $1/M$ to each of its components. This proves the stated comparison without assuming an optimal mixture is attained. The output has a finite-horizon feasibility tolerance, so its reward can exceed the exactly feasible optimum; the conclusion is a lower performance bound together with approximate feasibility.
\end{proof}

For a prescribed horizon $M$, using a constant $\eta=M^{-1/2}$ throughout that run and $\lambda_0=0$ gives feasibility tolerance at most $L/\sqrt M+G/M$ and reward tolerance $\delta+G^2/(2\sqrt M)$. This is a family of finite-horizon guarantees; it is not a proof for replacing $\eta$ by $k^{-1/2}$ within a single run. At fixed $\eta$, average feasibility is asymptotically exact and the reward tolerance retains $\delta+\eta G^2/2$.

\subsection{Sampled-cost cooperative guarantee}\label{app:proof-sampled}
The deterministic bound on \emph{every} multiplier cannot be asserted for random costs, even with an ideal oracle. The following replacement handles the reset-episode sampling used in Proposition~\ref{prop:average-epoch-certificate}.

\begin{theorem}[Cooperative oracle with unbiased sampled costs]\label{thm:sampled}
Assume the reward-range, strict-feasibility, and oracle conditions of Proposition~\ref{prop:coop-bound} hold almost surely. Use $\lambda_{k+1}=[\lambda_k+\eta(\widehat C_k-\varepsilon)]_+$ with deterministic $\lambda_0\geq0$, constant $\eta>0$, and $\mathbb E[\widehat C_k\mid\mathcal F_k]=C(\pi^k)$. Define
\begin{gather}\label{eq:drift-constants}
 D=\max\left\{\frac{2(R+\delta)}{\sigma},\eta G\right\},\qquad
 \theta=\frac{\sigma}{2\eta G^2},\qquad
 \rho=\exp\left(-\frac{\sigma^2}{8G^2}\right)<1,\\
 A=e^{\theta\lambda_0}+\frac{e^{\theta(D+\eta G)}}{1-\rho}.
\end{gather}
Then $\sup_k\mathbb E e^{\theta\lambda_k}\leq A$, $\lambda_M=O(\log M)$ almost surely, and both realized and conditional-expected average costs have limit superior at most $\varepsilon$ almost surely. For fixed $M$, with probability at least $1-\beta$,
\begin{equation}\label{eq:sampled-coop-feas}
 \bar C_M-\varepsilon\leq
 \frac{\theta^{-1}\log(2A/\beta)-\lambda_0}{\eta M}
 +\sqrt{\frac{2\log(2/\beta)}{M}}.
\end{equation}
The reward bound is
\begin{equation}\label{eq:sampled-reward}
 \mathbb E\bar J_{\Sigma,M}\geq J_{\rm mix}^{\star}
 -\delta-\frac{\eta G^2}{2}-\frac{\lambda_0^2}{2\eta M}.
\end{equation}
\end{theorem}
\begin{proof}
The oracle still compares expected Lagrangian values, even though the update uses an observed cost. Consequently the comparison with $\pi^\dagger$ in~\eqref{eq:oracle-drift} remains valid almost surely. On the event $\{\lambda_k>D\}$, division by $\lambda_k>0$ gives
\[
 g_k\leq\frac{R+\delta}{\lambda_k}-\sigma\leq-\frac{\sigma}{2},
\]
since $D\geq2(R+\delta)/\sigma$. Also $\lambda_k>D\geq\eta G$, and $\widehat g_k\geq-G$. Therefore $\lambda_k+\eta\widehat g_k>0$ on this event, so the next projection is inactive. Writing $X_k=\lambda_{k+1}-\lambda_k$, we obtain there
\[
 X_k=\eta\widehat g_k,\qquad |X_k|\leq\eta G,\qquad
 \mathbb E[X_k\mid\mathcal F_k]=\eta g_k\leq-\eta\sigma/2.
\]
This gives negative conditional drift when the multiplier is large, even though any particular sample can still increase it.

We next convert this drift into a bound on an exponential moment. Conditional on $\mathcal F_k$ and on $\lambda_k>D$, let $m_k=\mathbb E[X_k\mid\mathcal F_k]$. The conditional support of $X_k$ has length at most $2\eta G$; subtracting $m_k$ does not change that length. Conditional Hoeffding's lemma (Lemma~\ref{lem:conditional-hoeffding}, Appendix~\ref{app:probability-tools}) therefore gives
\begin{align*}
 \mathbb E[e^{\theta X_k}\mid\mathcal F_k]
 &=e^{\theta m_k}\mathbb E[e^{\theta(X_k-m_k)}\mid\mathcal F_k]\\
 &\leq\exp\{-\theta\eta\sigma/2+\theta^2\eta^2G^2/2\}\\
 &=\exp\{-\sigma^2/(8G^2)\}=\rho<1.
\end{align*}
The last equality follows by substituting $\theta=\sigma/(2\eta G^2)$: the two exponent terms are $-\sigma^2/(4G^2)$ and $\sigma^2/(8G^2)$.

We now translate this into a bound on the next multiplier. Since $\lambda_k$ is $\mathcal F_k$-measurable, the preceding display implies
\[
 \mathbb E[e^{\theta\lambda_{k+1}}\mid\mathcal F_k]
 =e^{\theta\lambda_k}\mathbb E[e^{\theta X_k}\mid\mathcal F_k]
 \leq\rho e^{\theta\lambda_k}
 \quad\text{on }\{\lambda_k>D\}.
\]
On the complementary event, $\lambda_{k+1}\leq\lambda_k+\eta G\leq D+\eta G$. Put $b=e^{\theta(D+\eta G)}$ and $Y_k=\mathbb E e^{\theta\lambda_k}$. Splitting over these two $\mathcal F_k$-measurable events and taking expectations yields
\begin{align*}
 Y_{k+1}
 &\leq\rho\,\mathbb E\!\left[e^{\theta\lambda_k}\mathbf1_{\{\lambda_k>D\}}\right]
       +b\Pr(\lambda_k\leq D)\\
 &\leq\rho Y_k+b.
\end{align*}
All these moments are finite at each finite $k$, since the bounded upward increment gives the deterministic bound $\lambda_k\leq\lambda_0+k\eta G$. Iterating the scalar inequality, starting at $Y_0=e^{\theta\lambda_0}$, gives
\[
 Y_k\leq\rho^k e^{\theta\lambda_0}
       +b\sum_{j=0}^{k-1}\rho^j
 \leq e^{\theta\lambda_0}+\frac{b}{1-\rho}=A.
\]
This proves the uniform exponential-moment bound.

For every $x\geq0$, Markov's inequality (Lemma~\ref{lem:markov}, Appendix~\ref{app:probability-tools}) now implies
\[
 \Pr(\lambda_M>x)\leq e^{-\theta x}\mathbb E e^{\theta\lambda_M}
 \leq A e^{-\theta x}.
\]
Taking $x=\theta^{-1}\log(A/\zeta)$ gives a tail probability at most $\zeta$ for any $\zeta\in(0,1)$. In particular, choose $\zeta=M^{-2}$ for $M\geq2$. The sum of these failure probabilities is finite, so the first Borel--Cantelli lemma (Lemma~\ref{lem:borel-cantelli}, Appendix~\ref{app:probability-tools}) implies that, almost surely, there is a finite random index after which
\[
 \lambda_M\leq\theta^{-1}\{\log A+2\log M\}.
\]
Thus $\lambda_M=O(\log M)$ almost surely. This is an eventual pathwise growth bound, not a deterministic upper bound on all multipliers.

We now invoke Proposition~\ref{prop:average-epoch-certificate} for asymptotic feasibility. Since $\log M/M\to0$, the required sublinear multiplier growth holds almost surely. The proposition's pathwise certificate proves realized average feasibility; its martingale argument and conditional unbiasedness prove the same limit-superior bound for $\bar C_M$.

For the finite-horizon assertion, use the multiplier tail bound with $\zeta=\beta/2$. With probability at least $1-\beta/2$,
\[
 \lambda_M\leq\theta^{-1}\log(2A/\beta).
\]
Independently of whether this event occurs, inequality~\eqref{eq:finite-sample-feas}, used with failure probability $\beta/2$, holds on an event of probability at least $1-\beta/2$. The union bound (Lemma~\ref{lem:union-bound}, Appendix~\ref{app:probability-tools}) ensures that both inequalities hold together with probability at least $1-\beta$. Substituting the multiplier bound into that inequality gives~\eqref{eq:sampled-coop-feas}. Here ``together'' requires no statistical independence: the probability of either failure is at most the sum of its two bounds.

It remains to prove the reward assertion. Fix a deterministic feasible finite mixture $\mu$. The same oracle comparison as in Proposition~\ref{prop:coop-bound} gives, almost surely at each epoch,
\[
 J(\pi^k)\geq J(\mu)-\delta+\lambda_kg_k.
\]
The energy lemma now concerns $\widehat g_k$, because those are the violations used in the update. To connect its conclusion to the preceding inequality, use measurability of $\lambda_k$, conditional unbiasedness, and the tower and pull-out identities in Lemma~\ref{lem:conditional-expectation} (Appendix~\ref{app:probability-tools}):
\begin{align*}
 \mathbb E[\lambda_k\widehat g_k]
 &=\mathbb E\!\left[\mathbb E[\lambda_k\widehat g_k\mid\mathcal F_k]\right]\\
 &=\mathbb E\!\left[\lambda_k\mathbb E[\widehat g_k\mid\mathcal F_k]\right]
 =\mathbb E[\lambda_kg_k].
\end{align*}
These products are integrable because $|g_k|,|\widehat g_k|\leq G$ and $\lambda_k\leq\lambda_0+k\eta G$. Applying Lemma~\ref{lem:energy} pathwise with $v_k=\widehat g_k$ and then taking expectations gives
\[
 \frac1M\sum_{k<M}\mathbb E[\lambda_kg_k]
 =\mathbb E\!\left[\frac1M\sum_{k<M}\lambda_k\widehat g_k\right]
 \geq-\frac{\lambda_0^2}{2\eta M}-\frac{\eta G^2}{2}.
\]
Averaging the oracle reward inequality and using this estimate yields
\[
 \mathbb E\bar J_{\Sigma,M}\geq J(\mu)-\delta
 -\frac{\lambda_0^2}{2\eta M}-\frac{\eta G^2}{2}.
\]
The error bound does not depend on $\mu$. Taking its supremum over feasible finite mixtures, or equivalently using arbitrarily near-optimal mixtures and passing to the limit, proves~\eqref{eq:sampled-reward}. No exchange of a random supremum and expectation is made: each comparator is fixed before the expectation is taken. The conclusion concerns expected reward, whereas the feasibility conclusions above hold almost surely.
\end{proof}

\subsection{What a full average-epoch Nash certificate requires}\label{app:proof-nash}
Definition~\ref{def:epoch-nash} is a constrained \emph{generalized} Nash concept: each player's feasible sequence depends on the opponents' sequence. It allows a different stationary deviation in every episode. A public, uniformly sampled epoch index implements the same average values, provided every agent uses the \emph{same} index and the deviator is allowed to condition on it. Independently sampling each agent's policy index gives a different joint distribution and can change both cost and reward. Restricting deviations to one fixed policy instead gives a constrained coarse-correlated type test; it does not establish the stronger definition used here. Neither definition describes strategic manipulation of the learning algorithm's future updates.

\noindent\textbf{Proof of Theorem~\ref{thm:nash}.}
\begin{proof}
Fix a horizon $M\geq1$, a player $i$, and any admissible unilateral deviation sequence $(\sigma_i^0,\ldots,\sigma_i^{M-1})$ in Definition~\ref{def:epoch-nash}. The opponents' sequence $(\pi_{-i}^k)_{k<M}$ is held fixed in this comparison. Write
\[
 g_k'=C(\sigma_i^k,\pi_{-i}^k)-\varepsilon,
 \qquad g_k=C(\pi^k)-\varepsilon.
\]
Admissibility means $M^{-1}\sum_{k<M}g_k'\leq0$, and bounded costs imply $|g_k'|\leq G$. Individual $g_k'$ may nevertheless be positive: the deviation only needs to respect its budget on average.

At every epoch, the $\delta$-Nash property for the priced game gives
\[
 J_i^\gamma(\pi^k)-\lambda_kg_k
 \geq J_i^\gamma(\sigma_i^k,\pi_{-i}^k)-\lambda_kg_k'-\delta.
\]
Rearranging isolates the unpriced reward gain:
\[
 J_i^\gamma(\sigma_i^k,\pi_{-i}^k)-J_i^\gamma(\pi^k)
 \leq\delta+\lambda_kg_k'-\lambda_kg_k.
\]
Denote the average of the left-hand side by $\mathcal D_{i,M}$. Averaging therefore gives
\[
 \mathcal D_{i,M}\leq\delta
 +\frac1M\sum_{k<M}\lambda_kg_k'
 -\frac1M\sum_{k<M}\lambda_kg_k.
\]
The last term is controlled by dual energy. To bound the comparator term, we must account for changing prices: the sign of $\sum_kg_k'$ does not determine the sign of $\sum_k\lambda_kg_k'$.

Choose any constant reference price $\ell\geq0$ and add and subtract it inside that sum. The resulting identity is
\[
 \frac1M\sum_{k<M}\lambda_kg_k'
 =\frac\ell M\sum_{k<M}g_k'
 +\frac1M\sum_{k<M}(\lambda_k-\ell)g_k'.
\]
The first term is nonpositive because the comparator is average-feasible and $\ell\geq0$. For each summand of the second term,
\[
 (\lambda_k-\ell)g_k'\leq|\lambda_k-\ell|\,|g_k'|
 \leq G|\lambda_k-\ell|.
\]
Consequently
\[
 \frac1M\sum_{k<M}\lambda_kg_k'
 \leq\frac G M\sum_{k<M}|\lambda_k-\ell|.
\]
This upper bound depends only on the observed price sequence and the universal cost bound, not on the chosen player or deviation.

We now apply Lemma~\ref{lem:energy} with $v_k=g_k$, as required by the expected-cost update in the theorem. It gives
\[
 -\frac1M\sum_{k<M}\lambda_kg_k
 \leq\frac{\lambda_0^2}{2\eta M}+\frac{\eta G^2}{2}.
\]
Combining the two bounds, for every $\ell\geq0$ we have
\[
 \mathcal D_{i,M}\leq\delta+\frac{\lambda_0^2}{2\eta M}
 +\frac{\eta G^2}{2}+\frac G M\sum_{k<M}|\lambda_k-\ell|.
\]
Taking the infimum over $\ell\geq0$ proves $\mathcal D_{i,M}\leq\delta_M$. This passage is valid even without choosing a minimizer, since the left-hand side is bounded by every member of the displayed family of upper bounds. The player and admissible deviation were arbitrary, and the resulting bound is uniform over both. Thus every average-feasible unilateral deviation satisfies the required reward inequality. This is a deterministic argument for the sequence under consideration; no union bound over policy deviations is needed.

Finally, we establish the feasibility part of the definition. Proposition~\ref{prop:average-epoch-certificate}, applied to the expected costs used in this update, yields
\[
 \bar C_M-\varepsilon\leq\frac{\lambda_M-\lambda_0}{\eta M}
 \leq\left[\frac{\lambda_M-\lambda_0}{\eta M}\right]_+=\alpha_M.
\]
The positive part ensures a nonnegative reported tolerance, including when the sequence has net slack and the multiplier decreases. Together with the uniform deviation bound, this proves the $(\alpha_M,\delta_M)$-average-epoch constrained Nash certificate. The proof supplies no convergence claim for $\alpha_M$ or $\Delta_M$; such a claim would require additional information about the price sequence and oracle errors.
\end{proof}

For a comparator feasible at every epoch, $g_k'\leq0$ directly eliminates the dispersion term. For average feasibility alone, it cannot be discarded: violations $(-a,+a)$ at prices $(0,L)$ have zero unweighted mean but price-weighted mean $La/2>0$. Bounded multipliers need not have $\Delta_M\to0$; sustained oscillations can leave a positive dispersion term. A family of horizon-$M$ runs with $\eta=M^{-1/2}$, bounded prices, vanishing oracle errors, and $\Delta_M\to0$ would give vanishing Nash and feasibility tolerances. These are sufficient conditions, not properties established for the reported learners.

An alternative exact certificate is useful when a common price can be found. If every epoch is a Lagrangian Nash profile at the \emph{same} price $\ell\geq0$, $\bar g_M\leq0$, and $\ell\bar g_M=0$, summing the same oracle inequalities gives deviation gain at most $\ell\bar g'_M-\ell\bar g_M\leq0$. Thus the sequence is an exact constrained Nash profile. This proves sufficiency of common-price equilibrium, feasibility, and complementary slackness without asserting their existence or convergence. Different players can also use different constant supporting prices; one shared price selects a more restricted class.

\subsection{Why joint feasibility alone cannot control a game oracle}
\begin{lemma}[A jointly feasible game with a permanently unsafe Nash oracle]\label{lem:game-obstruction}
An exact Nash oracle and a jointly strictly feasible policy do not, in general, imply feasibility of the projected-dual game method.
\end{lemma}
\begin{proof}
Consider a one-state, one-step game with two players and actions $S$ and $H$. Player $i$ receives $r_i(a_i,a_{-i})=\mathbf1\{a_i=H\}$. The common cost is zero at $(S,S)$ and one at each of the other three action profiles. Set $\varepsilon=1/2$. The joint policy $(S,S)$ has cost zero, so it satisfies strict feasibility with margin $\sigma=1/2$. All rewards and costs are bounded.

Fix any price $\lambda\geq0$. At $(H,H)$, each player receives Lagrangian payoff $1-\lambda(1-1/2)=1-\lambda/2$. If one player alone changes its action to $S$, the other still chooses $H$. The common cost therefore remains one, whereas the deviator's reward falls to zero. Its new Lagrangian payoff is $-\lambda/2$, exactly one less than before. This comparison holds for every price, however large.

To include randomized deviations, let the deviating player choose $H$ with probability $p\in[0,1]$ against an opponent who chooses $H$ surely. Cost remains one surely and expected reward is $p$, so the deviating payoff is $p-\lambda/2\leq1-\lambda/2$, with equality only at $p=1$. Thus $(H,H)$ is a strict Nash profile at every price, including against all randomized unilateral alternatives.

An exact Nash oracle is consequently permitted to select $(H,H)$ at every epoch. For that selection $g_k=1/2$, and the nonnegative multiplier update has no active projection:
\[
 \lambda_{k+1}=\lambda_k+\eta/2,
 \qquad \lambda_k=\lambda_0+k\eta/2.
\]
The second identity follows by induction from the first. The average cost is exactly one for every horizon, which exceeds the budget $1/2$. Hence bounded rewards, joint strict feasibility, and an exact Nash oracle do not force the method to be feasible. The obstruction is unilateral: one player cannot reduce the common cost while its opponent chooses $H$, although both can reduce it by changing together.
\end{proof}

A stronger \emph{uniform unilateral} Slater condition would restore multiplier control for the scalar game. Suppose there is a common $\sigma>0$ such that, for every player $i$ and opponents' policy $\pi_{-i}$, some admissible $z_i$ satisfies $g(z_i,\pi_{-i})\leq-\sigma$. At epoch $k$, the Nash oracle's comparison with that alternative gives
\[
 J_i^\gamma(\pi^k)-\lambda_kg_k
 \geq J_i^\gamma(z_i,\pi_{-i}^k)+\lambda_k\sigma-\delta.
\]
If each player's reward range is at most $R$, rearrangement yields $\lambda_kg_k\leq R+\delta-\lambda_k\sigma$, exactly the drift inequality used above. For expected-cost updates the same threshold and induction prove bounded multipliers. For conditionally unbiased sampled costs the same negative-drift and exponential-moment argument proves almost-sure sublinear growth. Proposition~\ref{prop:average-epoch-certificate} then proves feasibility in either case. This argument controls cost only: the $G\Delta_M$ term for average-feasible strategic deviations still requires separate control. Appendix~\ref{app:environment} establishes why the unilateral assumption fails in this fishery.

\subsection{The implemented PI controller and shaping error}\label{app:mismatch}
For the capped PI recursion~\eqref{eq:pi}, let $\widehat C_k=1$ for all $k$. Eventually $I_k=I_{\max}$ and $\lambda_k=I_{\max}+K_P(1-\varepsilon)$, yet the average violation remains $1-\varepsilon>0$. Consequently bounded PI multipliers do not provide Proposition~\ref{prop:average-epoch-certificate}'s certificate.

The second gap is the actor objective. Let $S(\pi)$ be the expected discounted shaped cost actually penalized by a particular actor. At a fixed $\lambda\geq0$, define $F(\pi)=J(\pi)-\lambda C(\pi)$ and $\widetilde F(\pi)=J(\pi)-\lambda S(\pi)$. Suppose $|S(\pi)-C(\pi)|\leq e$ uniformly over the comparison class. Then $|F(\pi)-\widetilde F(\pi)|\leq\lambda e$. If an output $\pi$ is $\delta$-optimal for $\widetilde F$, comparison with any alternative $z$ gives
\[
 F(\pi)\geq\widetilde F(\pi)-\lambda e
 \geq\widetilde F(z)-\delta-\lambda e
 \geq F(z)-\delta-2\lambda e.
\]
Taking a supremum over $z$ proves a $(\delta+2\lambda e)$ oracle tolerance for the terminal-cost objective. The constant $\lambda\varepsilon$ in the Lagrangian cancels from policy comparisons. The same inequalities apply to unilateral alternatives with the opponents fixed if the discrepancy bound is uniform over those alternatives. 

With discounted episodic reward, an exact terminal-cost implementation would use the shared terminal penalty $-\lambda\gamma^{-(H-1)}(1-B_H/K)$ at transition $H-1$ and zero cost penalties earlier. Its discounted sum is exactly $-\lambda C$ in expectation. Each self-interested agent must receive that global penalty to implement~\eqref{eq:lagrangians}; catch attribution changes the game. The policy-independent $\lambda\varepsilon$ term can be omitted from actor updates. 

\section{Properties of the Gordon--Schaefer environment}\label{app:environment}

This section connects the fishery model to the assumptions and interpretation of the policy-sequence results. The calculations are elementary consequences of the specified growth and harvesting rules; we use them as supporting analysis, rather than as new ecological results. They answer three questions needed by the paper: whether costs and rewards obey the bounds used in the proofs, whether preserving the stock is attainable jointly and unilaterally, and what the depletion signal measures. The resulting distinctions explain why the cooperative guarantee applies under different conditions from the game guarantee and why the experimental curves report terminal depletion separately from the training penalty.

The first lemma supplies concrete model bounds for Propositions~\ref{prop:average-epoch-certificate} and~\ref{prop:coop-bound} and Theorem~\ref{thm:sampled}. Write $x=B/K$ and $a=\min\{q\sum_i e_i,1\}$ for the fraction harvested. The post-harvest stock is $y=(1-a)x$ and, for $0\leq r\leq1$, the next normalized stock is
\[
 F_a(x)=(1-a)x\{1+r-r(1-a)x\}.
\]
This formula incorporates proportional catch scaling when requested catch exhausts the stock.

\begin{lemma}[Invariant stock interval and reward bounds]\label{lem:stock}
For $0\leq r\leq1$, the unclipped map already takes $[0,1]$ into $[0,1]$. Zero stock is absorbing. Joint zero effort preserves the initial state $B_0=K$ exactly. Per-step total catch lies in $[0,K]$, so the range of discounted team return is at most
$K\sum_{t=0}^{H-1}\gamma^t$, with the value $KH$ when $\gamma=1$.
\end{lemma}
\begin{proof}
Normalize biomass by $K$ and let $y$ be the post-harvest stock fraction. Proportional catch scaling ensures that total catch cannot exceed available biomass, so $0\leq y\leq x\leq1$ whenever the pre-harvest fraction $x$ lies in $[0,1]$. Before clipping, growth maps $y$ to
\[
 f(y)=y+ry(1-y),\qquad f'(y)=1+r-2ry.
\]
For $0\leq r\leq1$ and $y\in[0,1]$, $f'(y)\geq1-r\geq0$. Thus $f$ is nondecreasing on this interval. Since $f(0)=0$ and $f(1)=1$, it follows that $0\leq f(y)\leq1$. The unclipped transition already lies in $[0,1]$, so clipping does not change it. Starting in this interval and applying the same argument at each transition proves invariance by induction.

At zero stock, requested catches are zero and $f(0)=0$. Hence zero is absorbing until the external episode reset. Under joint zero effort, no biomass is removed. In particular, starting at $x_0=1$ gives $x_1=f(1)=1$, and induction gives $x_t=1$ throughout the episode. Its terminal depletion is therefore exactly zero.

For the reward assertion, let $h_t=\sum_i h_{i,t}$ be total catch. Scaling and nonnegativity give $0\leq h_t\leq B_t\leq K$ on every trajectory. For $0\leq\gamma\leq1$,
\[
 0\leq\sum_{t=0}^{H-1}\gamma^t h_t
 \leq K\sum_{t=0}^{H-1}\gamma^t.
\]
Taking expectations preserves both inequalities. Thus every team policy value lies in this common interval, and the difference between its supremum and infimum is at most the interval's length. For $\gamma<1$ that length is $K(1-\gamma^H)/(1-\gamma)$; for $\gamma=1$ it is $KH$. This establishes a valid reward-range bound without assuming any optimal policy exists.
\end{proof}

The invariant interval gives $C,\widehat C\in[0,1]$, as required by the feasibility and concentration arguments. The return bound provides the concrete choice $R=K\sum_{t=0}^{H-1}\gamma^t$ in the cooperative multiplier and reward bounds. For any positive budget, joint abstention has $C=0$ and supplies the strict-feasibility margin $\sigma=\varepsilon$, when abstention belongs to the chosen policy class. Thus the resource requirement is attainable by coordinated agents; the cooperative theorem then quantifies how much reward can be retained while meeting it. These calculations verify the environmental assumptions, while the theorem separately requires the stated solver accuracy. At budget zero, strict feasibility is impossible because $C\geq0$.

The next lemma explains the physical mechanism behind depletion and prepares the distinction between joint and unilateral feasibility. Sustained extraction can exceed regeneration even though the unharvested resource renews itself. The constant-effort threshold quantifies this balance, and the positive fixed point describes the stock sustained below that threshold. A positive stock alone does not specify whether a chosen depletion budget is met; the budget imposes a quantitative preservation requirement.

\begin{lemma}[Collapse threshold under constant total effort]\label{lem:collapse}
Let the harvested fraction $a\in[0,1]$ be fixed. If $d=(1-a)(1+r)<1$, then $x_H\leq d^H x_0$. If $d>1$, the map has a positive fixed point
\[
 x^*=\frac{(1-a)(1+r)-1}{r(1-a)^2}
\]
when $r>0$ and $a<1$; for $0\leq r\leq1$ this point lies in $(0,1]$ and is locally stable. These are properties of constant effort, not a characterization of learned equilibrium policies.
\end{lemma}
\begin{proof}
For fixed harvested fraction $a$, expand the transition as
\[
 F_a(x)=dx-bx^2,\qquad d=(1-a)(1+r),\qquad b=r(1-a)^2\geq0.
\]
On $[0,1]$, the quadratic term is nonnegative before subtraction, so $0\leq F_a(x)\leq dx$. If $d<1$, applying this inequality repeatedly gives $x_1\leq dx_0$, $x_2\leq dx_1\leq d^2x_0$, and, by induction, $x_H\leq d^Hx_0$. This proves geometric decay in the strict subcritical case. In particular, $a=1$ gives $d=b=0$ and immediate extinction.

Now suppose $d>1$. This implies $r>0$ and $a<1$, hence $b>0$. A positive fixed point solves $x=dx-bx^2$. Dividing by $x>0$ gives $1=d-bx$, so the unique positive candidate is
\[
 x^*=(d-1)/b=\frac{(1-a)(1+r)-1}{r(1-a)^2}.
\]
Its numerator and denominator are positive. To verify that the candidate is in the invariant stock interval, use Lemma~\ref{lem:stock}: $F_a(1)=f(1-a)\leq1$. Since $F_a(1)=d-b$, this implies $d-1\leq b$, and hence $x^*\leq1$.

For local stability, differentiation gives $F_a'(x)=d-2bx$. Substitution of $bx^*=d-1$ yields
\[
 F_a'(x^*)=d-2(d-1)=2-d.
\]
The assumptions imply $1<d\leq1+r\leq2$, so $|F_a'(x^*)|<1$. More explicitly, continuity of the derivative supplies a neighborhood of $x^*$ and a number $c<1$ on which $|F_a'|\leq c$. The mean-value theorem then gives $|F_a(x)-x^*|\leq c|x-x^*|$ for sufficiently close $x\in[0,1]$. This inequality keeps subsequent iterates in that neighborhood and makes their distance to $x^*$ decrease geometrically. It also covers the endpoint case $x^*=1$ by restricting to the invariant interval. Thus the fixed point is locally asymptotically stable.

At the boundary $d=1$ with $b>0$, the formula $(d-1)/b$ gives zero, not a positive fixed point; the strict geometric bound for $d<1$ makes no claim there. If also $b=0$, then $d=1$ forces $r=0$ and $a=0$, and $F_a(x)=x$: every stock is fixed. These boundary cases are separate from the two strict regimes asserted in the lemma.
\end{proof}

The strict collapse condition is $a>r/(1+r)$. At $r=.3$, a single agent's maximal effort with $q=.5$ already gives $a\geq.5>r/(1+r)$. More generally, if the opponent always exerts unit effort, every possible response of the other agent leaves post-harvest stock at most $B_t/2$. Thus, directly from $f(y)\leq(1+r)y$,
\begin{equation}\label{eq:unilateral-collapse}
 B_{t+1}\leq .65 B_t,\qquad C\geq1-.65^{60}.
\end{equation}
This holds pathwise for randomized and history-dependent responses as well. Since $.65^{60}<6\times10^{-12}$, no unilateral response can satisfy any of the tested budgets against that opponent. The game source's uniform unilateral Slater condition fails, whereas joint Slater holds. If both agents exert unit effort at the first step, $q(e_1+e_2)=1$, so all stock is removed and zero remains absorbing until the episode reset.

This unilateral bound is the main reason for including the collapse calculation. Coordinated abstention provides a feasible policy, yet an individual harvester cannot preserve the stock against every possible opponent policy. The cooperative strict-feasibility assumption therefore has an explicit witness, while the stronger uniform unilateral assumption of the earlier game result fails in the unrestricted effort class. This motivates treating resource feasibility and the Nash deviation bound separately in the present framework. The calculation identifies a structural possibility in the environment; the experiments examine the behavior of the learned policies.

The following identity explains the choice of evaluation metric. A training penalty based on every downward stock movement counts depletion before later regeneration, whereas terminal depletion records the net loss remaining at episode end. The identity gives the exact difference between these two undiscounted quantities, showing why budget satisfaction is evaluated using the endpoint cost that appears in the formulation.

\begin{lemma}[One-sided loss is not terminal depletion]\label{lem:shaping}
For any stock path starting at $B_0=K$, define
$D=K^{-1}\sum_{t<H}[B_t-B_{t+1}]_+$ and
$U=K^{-1}\sum_{t<H}[B_{t+1}-B_t]_+$. Then
\begin{equation}\label{eq:loss-identity}
 D=1-B_H/K+U\geq1-B_H/K.
\end{equation}
\end{lemma}
\begin{proof}
For a real increment $z$, its positive and negative parts satisfy $[-z]_+-[z]_+=-z$. Indeed, for $z\geq0$ the left-hand side is $0-z$, while for $z<0$ it is $-z-0$. Apply this identity to $z_t=B_{t+1}-B_t$ at each transition. The definitions of $D$ and $U$ give
\begin{align*}
 K(D-U)
 &=\sum_{t=0}^{H-1}\bigl([B_t-B_{t+1}]_+-[B_{t+1}-B_t]_+\bigr)\\
 &=\sum_{t=0}^{H-1}(B_t-B_{t+1})=B_0-B_H.
\end{align*}
The last equality is telescoping. Since $B_0=K>0$, division by $K$ yields $D-U=1-B_H/K$, or $D=1-B_H/K+U$. Every summand defining $U$ is nonnegative, so $D\geq1-B_H/K$. Equality holds exactly when there is no positive stock increment along the finite path. This is a pathwise, undiscounted identity and therefore does not require assumptions about how the actions were selected.
\end{proof}
Undiscounted one-sided loss is therefore an upper bound on terminal depletion, with a gap equal to accumulated positive stock increments. Its \emph{discounted}, scaled, per-agent attribution is a different quantity and has no corresponding equality with the global endpoint cost. Terminal depletion alone does not record the lowest stock reached during an episode.

\subsection{Why policy sequences can help, and what that does not establish}
The final observation links the environmental analysis back to the nonstationary solution concept. Once policies can have different harvest and depletion values, the average-epoch formulation permits the resource budget to be allocated across them. This is the role of policy mixtures in the cooperative reward comparison of Proposition~\ref{prop:coop-bound}. The example and construction below clarify that role and distinguish it from a claim that changing policies is always better than stationary randomized control.

With reset episodes, a mixture of \emph{whole} joint policies has reward and cost equal to the weighted averages of their individual values. Mixing low-reward, low-cost policies with high-reward, high-cost policies yields intermediate reward and cost values. For example, take a one-step two-action control problem with $(C,J)=(0,0)$ and $(1,2)$ and budget $1/2$. Among deterministic policies only the first is feasible; alternating the two policies yields $(\bar C,\bar J)=(1/2,1)$. This is a strict improvement over the deterministic stationary class, but a stationary policy randomized equally between the actions attains the same pair. Thus executing different policies across episodes expands a restricted policy class and avoids requiring each component to be feasible; it does not imply that nonstationarity always outperforms unrestricted randomized stationary control.

In a finite fully observed finite-horizon cooperative MDP, the usual occupancy constraints make this distinction exact. Given an episode mixture, let $d_t(s,a)$ be its joint state--action marginal. These marginals obey the initial-distribution and flow equations because those equations are linear. Define a joint Markov policy by $\pi_t(a\mid s)=d_t(s,a)/\sum_b d_t(s,b)$ wherever the denominator is positive, choosing arbitrarily at unvisited states. Induction on $t$ proves that this policy has the same marginals, hence the same additive rewards and terminal-state distribution. The time-dependent Markov policy is stationary on $(s,t)$. A joint distribution over actions need not factor into independent local actors; decentralized information or a restricted neural policy class can prevent this construction. We therefore motivate policy sequences as the natural output of the Lagrangian procedure, without asserting an unproved advantage over all stationary randomized policies in this fishery.

\section{Lagrangian, IPPO, and MAPPO algorithms}\label{app:algorithms}

\subsection{The outer primal--dual policy-sequence method}
Algorithm~\ref{alg:primal-dual} states the oracle method separately from the practical PPO--PI learners. The multiplier update raises the price when depletion exceeds the budget and lowers it when depletion falls below the budget, subject to remaining nonnegative. In the cooperative formulation this has a dual-subgradient interpretation; in the game it generates the prices for successive Lagrangian games. At each price, the solver seeks a team optimum or a Nash equilibrium, respectively. Solving one constrained problem is thereby replaced by a sequence of unconstrained games or control problems indexed by price; this reduction does not make their inner solution computationally trivial.

The average-reward version follows the epoch mechanism of~\citet{das2025cmg,das2025lagrangian}, with the upper-bound sign convention. Its stage reward is $r_i(s,a)-\lambda_k(c(s,a)-\varepsilon)$ and its value is an infinite-horizon time average. A finite block from a general starting state need not be an unbiased estimate of that value. Exact equality is a substantive oracle/evaluation assumption in the source results; choosing a large $T_0$ alone does not prove it. If a uniform bias estimate is available, the bias term in the proof of Proposition~\ref{prop:average-epoch-certificate} quantifies its contribution. Different block lengths require time-weighted averages to describe physical-time performance.

For reset episodes, a long-run regenerative interpretation is available: repeat independent episodes of fixed length $H$. The cycle-average reward is $H^{-1}\mathbb E\sum_{t<H}R_t$; the endpoint cost averaged over cycles is $C$. A per-time-step cost which is zero except for $H(1-B_H/K)$ on the last transition has the same time-average value $C$. This equivalence relies on fixed cycle length and reset, and uses \emph{undiscounted} reward. It does not convert $J_i^\gamma$ with $\gamma=.99$ into average physical-time catch, nor does it describe a fishery with no stock reset.

A PPO inner learner for a continuing average-reward problem would require a corresponding value estimator. With fixed price, let $\widetilde r_t=r_t-\lambda_k(c_t-\varepsilon)$, let $\widehat\rho$ estimate its average reward, and let $h_\phi$ estimate a differential value. A possible differential TD residual is
\[
 d_t^{\rm av}=\widetilde r_t-\widehat\rho+h_\phi(s_{t+1})-h_\phi(s_t).
\]
One can form truncated advantage estimates $\sum_{l\geq0}\lambda_{\rm GAE}^l d_{t+l}^{\rm av}$ and apply the clipped policy update below. \editorialflag{This differential-value estimator is appropriate to the continuing average-reward formulation. The episodic experiments instead use discounted GAE, defined in~\eqref{eq:gae-appendix}.} The differential Bellman equation and stable estimation require appropriate recurrence assumptions.

\subsection{Common discounted PPO computations}
The practical algorithms use PPO~\citep{schulman2017proximal} with generalized advantage estimation~\citep{schulman2016gae}. On transition $t=0,\ldots,H-1$, write
\[
 g_t=\frac{[B_t-B_{t+1}]_+}{K},\quad
 a_{i,t}^{\rm cost}=\begin{cases}h_{i,t}/\sum_jh_{j,t},&\sum_jh_{j,t}>0,\\0,&\sum_jh_{j,t}=0.\end{cases}
\]
This indexing is equivalent to the one-sided loss defined in Appendix~\ref{app:protocol}, where the transition is indexed by its endpoint. The implemented PPO reward divides physical catch by $K$, while the theoretical reward retains the definition in Section~\ref{sec:formulation}. Fixing the epoch price gives
\begin{equation}\label{eq:shaped-signals}
 \widetilde R_{i,t}=h_{i,t}/K-w\lambda_kg_ta_{i,t}^{\rm cost}
 \quad\text{(IPPO)},\qquad
 \widetilde R_{\Sigma,t}=\sum_i h_{i,t}/K-w\lambda_kg_t
 \quad\text{(MAPPO)}.
\end{equation}
For an appropriate critic input $x_t$, use the frozen rollout critic to compute
\begin{align}\label{eq:gae-appendix}
 d_t&=\widetilde R_t+\gamma V_{\phi_{\rm old}}(x_{t+1})-V_{\phi_{\rm old}}(x_t),&
 \widehat A_t&=\sum_{l=0}^{H-1-t}(\gamma\lambda_{\rm GAE})^l d_{t+l},\\
 \widehat V_t&=\widehat A_t+V_{\phi_{\rm old}}(x_t),&
 r_{i,t}(\theta_i)&=\frac{\pi_{\theta_i}(e_{i,t}\mid o_{i,t})}
 {\pi_{\theta_{i,\rm old}}(e_{i,t}\mid o_{i,t})}.
\end{align}
At the true episode boundary set $V(x_H)=0$ and do not continue the advantage recursion across reset. A rollout truncated before that boundary instead bootstraps the critic. For continuous actions, the ratio uses probability \emph{densities}, including the squash change-of-variables consistently in both evaluations. PPO maximizes
\begin{equation}\label{eq:ppo-appendix}
 L_i^{\rm clip}(\theta_i)=\widehat{\mathbb E}_t
 \min\{r_{i,t}(\theta_i)\widehat A_t,
 \operatorname{clip}(r_{i,t}(\theta_i),1-\epsilon_{\rm clip},1+\epsilon_{\rm clip})\widehat A_t\}
 +\beta_{\rm ent}\widehat{\mathbb E}_t\mathcal H(\pi_{\theta_i}),
\end{equation}
and fits its critic to the fixed targets $\widehat V_t$ by squared error. If entropy regularization is omitted, set $\beta_{\rm ent}=0$. The current implementation uses affine Gaussian means with state-independent log standard deviations, followed by a logistic squash to obtain effort, and affine value functions. Updates use full-rollout gradients with fixed learning rates, not an adaptive-moment optimizer. \editorialflag{Each epoch collects one $H=60$-step episode and performs five PPO passes. The actor and critic learning rates are $0.03$ and $0.06$, respectively, and the entropy coefficient is zero.} In MAPPO, the centralized critic is fitted before the actor passes; IPPO updates its local critics during those passes. The actor score contributions are clipped before averaging.

The standard-deviation clamp $[0.06,0.8]$ applies to the latent Gaussian, not to the squashed effort distribution. A positive latent variance does not prevent effort from concentrating near an action boundary when the mean saturates. Before the PPO passes, advantages are centered and divided by the larger of one and their empirical standard deviation. This avoids amplifying a small-variance batch, but it does not guarantee invariance of the response to the multiplier. In particular, scaling a combined reward-and-cost advantage need not remove its dependence on the price.

In Algorithms~\ref{alg:ippo}--\ref{alg:mappo}, $N_{\rm ep}$ is the number of complete reset episodes per rollout batch and $E$ the number of PPO passes. The operator $\mathsf{Opt}_{\pi}$ takes an ascent step and $\mathsf{Opt}_V$ a descent step using the supplied optimizer settings; plain gradient updates are $\theta\gets\theta+a_\pi\nabla L$ and $\phi\gets\phi-a_V\nabla L_V$. The GAE operator is exactly~\eqref{eq:gae-appendix}, independently applied to each episode with zero terminal bootstrap. The minibatch objective is~\eqref{eq:ppo-appendix}; IPPO uses $\widehat A_{i,b,t}$ and MAPPO uses the same $\widehat A_{b,t}$ for every actor. Old densities, advantages, and value targets remain fixed throughout all $E$ passes. The catch shares $a_{i,b,t}^{\rm cost}$ are defined above, with zero share when total catch is zero. Set $z=1$ for constrained training and $z=0$ for the unpriced baseline. Stored rollout records include observations, actions, old log densities, catches, and stock transitions. \editorialflag{The implementation logs training metrics, within-episode diagnostics, and mean-effort summaries at a fixed reference state. These summaries describe training behavior rather than storing the complete sequence of actor parameters.}

\subsection{Independent PPO with the shared ecological price}
IPPO~\citep{dewitt2020ippo} has independently trained actors and local critics. Here the individual reward is catch, so the agents are self-interested. Independence of training architecture alone does not imply self-interest; IPPO can also be trained on a shared reward. Our implementation uses actor observation $o_{i,t}=(B_t/K,t/H,h_{i,t-1}/K)$, a logistic-squashed Gaussian actor with affine mean, and a local linear critic. A policy stationary in this augmented observation can depend on physical time through $t/H$; it need not be stationary in biomass alone.

\begin{algorithm}[t]
\caption{IPPO with terminal-cost PI feedback}\label{alg:ippo}
\small
\begin{algorithmic}[1]
\Require Initial actor/critic parameters $(\theta,\phi)$; $M,H,N_{\rm ep},E\in\mathbb N$;
$\varepsilon,w,\gamma,\lambda_{\rm GAE},\epsilon_{\rm clip},\beta_{\rm ent},K_P,K_I,I_{\max}$;
minibatch partition, optimizers $\mathsf{Opt}_{\pi},\mathsf{Opt}_V$; flag $z\in\{0,1\}$
\State $(I_0,\lambda_0)\gets(0,0)$
\For{$k=0,\ldots,M-1$}
 \State $(\theta^{\rm old},\phi^{\rm old})\gets(\theta,\phi)$; $\boldsymbol\pi^k\gets\prod_i\pi_{\theta_i^{\rm old}}$
 \For{$b=1,\ldots,N_{\rm ep}$}
  \State $B_{b,0}\gets K$; $h_{i,b,-1}\gets0$ for all $i$
  \For{$t=0,\ldots,H-1$}
   \State $e_{i,b,t}\sim\pi_{\theta_i^{\rm old}}(\cdot\mid o_{i,b,t})$ for all $i$
   \State $(h_{1:n,b,t},B_{b,t+1})\gets\operatorname{Fishery}(B_{b,t},e_{1:n,b,t})$
   \State $g_{b,t}\gets[B_{b,t}-B_{b,t+1}]_+/K$; $\widetilde R_{i,b,t}\gets h_{i,b,t}/K-w\lambda_k g_{b,t}a_{i,b,t}^{\rm cost}$ for all $i$
  \EndFor
  \State $(\widehat A_{i,b,t},\widehat V_{i,b,t})_{t=0}^{H-1}\gets\operatorname{GAE}(\widetilde R_{i,b,0:H-1},V_{\phi_i^{\rm old}},o_{i,b,0:H})$ for all $i$
 \EndFor
 \For{$p=1,\ldots,E$}
  \For{each minibatch $\mathcal B$ of the collected transitions}
   \State $\theta_i\gets\mathsf{Opt}_{\pi}(\theta_i,\nabla_{\theta_i}L^{\rm clip}_{i,\mathcal B}(\theta_i))$ for all $i$
   \State $\phi_i\gets\mathsf{Opt}_{V}(\phi_i,\nabla_{\phi_i}\widehat{\mathbb E}_{\mathcal B}(V_{\phi_i}(o_{i,b,t})-\widehat V_{i,b,t})^2)$ for all $i$
  \EndFor
 \EndFor
 \State $\widehat C_k\gets N_{\rm ep}^{-1}\sum_b(1-B_{b,H}/K)$; $u_k\gets\widehat C_k-\varepsilon$
 \State $I_{k+1}\gets z\operatorname{clip}_{[0,I_{\max}]}(I_k+K_Iu_k)$
 \State $\lambda_{k+1}\gets z[K_Pu_k+I_{k+1}]_+$
\EndFor
\State \Return $\theta$, $\{\boldsymbol\pi^k,\widehat C_k,\lambda_k\}_{k<M}$
\end{algorithmic}
\end{algorithm}

The controller observes one global terminal cost, even though learning uses individual shaped rewards. No opponent policy or opponent critic is supplied to the local PPO update. Shared feedback is therefore an ecological coordination signal, not a centralized optimization of individual rewards. As both agents learn, each agent faces a changing environment; clipping alone does not certify a Nash best response.

\subsection{MAPPO with the cooperative ecological price}
MAPPO~\citep{yu2022surprising} uses centralized training with decentralized execution. In the manuscript's cooperative arm, every actor uses the same team advantage, and a centralized critic uses \editorialflag{the shared input $x_t=(B_t/K,t/H)$}. The actors execute their own policies. The experiment changes both the critic architecture and the reward objective relative to IPPO, so performance differences cannot be attributed solely to centralizing the critic.

\begin{algorithm}[t]
\caption{Cooperative MAPPO with terminal-cost PI feedback}\label{alg:mappo}
\small
\begin{algorithmic}[1]
\Require Initial actor/critic parameters $(\theta,\phi)$; $M,H,N_{\rm ep},E\in\mathbb N$;
$\varepsilon,w,\gamma,\lambda_{\rm GAE},\epsilon_{\rm clip},\beta_{\rm ent},K_P,K_I,I_{\max}$;
minibatch partition, optimizers $\mathsf{Opt}_{\pi},\mathsf{Opt}_V$; flag $z\in\{0,1\}$
\State $(I_0,\lambda_0)\gets(0,0)$
\For{$k=0,\ldots,M-1$}
 \State $(\theta^{\rm old},\phi^{\rm old})\gets(\theta,\phi)$; $\boldsymbol\pi^k\gets\prod_i\pi_{\theta_i^{\rm old}}$
 \For{$b=1,\ldots,N_{\rm ep}$}
  \State $B_{b,0}\gets K$; $h_{i,b,-1}\gets0$ for all $i$
  \For{$t=0,\ldots,H-1$}
   \State $e_{i,b,t}\sim\pi_{\theta_i^{\rm old}}(\cdot\mid o_{i,b,t})$ for all $i$
   \State $(h_{1:n,b,t},B_{b,t+1})\gets\operatorname{Fishery}(B_{b,t},e_{1:n,b,t})$
   \State $g_{b,t}\gets[B_{b,t}-B_{b,t+1}]_+/K$; $\widetilde R_{\Sigma,b,t}\gets\sum_i h_{i,b,t}/K-w\lambda_k g_{b,t}$
  \EndFor
  \State $(\widehat A_{b,t},\widehat V_{b,t})_{t=0}^{H-1}\gets\operatorname{GAE}(\widetilde R_{\Sigma,b,0:H-1},V_{\phi^{\rm old}},x_{b,0:H})$
 \EndFor
 \For{$p=1,\ldots,E$}
  \For{each minibatch $\mathcal B$ of the collected transitions}
   \State $\theta_i\gets\mathsf{Opt}_{\pi}(\theta_i,\nabla_{\theta_i}L^{\rm clip}_{i,\mathcal B}(\theta_i))$ for all $i$
   \State $\phi\gets\mathsf{Opt}_{V}(\phi,\nabla_\phi\widehat{\mathbb E}_{\mathcal B}(V_\phi(x_{b,t})-\widehat V_{b,t})^2)$
  \EndFor
 \EndFor
 \State $\widehat C_k\gets N_{\rm ep}^{-1}\sum_b(1-B_{b,H}/K)$; $u_k\gets\widehat C_k-\varepsilon$
 \State $I_{k+1}\gets z\operatorname{clip}_{[0,I_{\max}]}(I_k+K_Iu_k)$
 \State $\lambda_{k+1}\gets z[K_Pu_k+I_{k+1}]_+$
\EndFor
\State \Return $\theta$, $\{\boldsymbol\pi^k,\widehat C_k,\lambda_k\}_{k<M}$
\end{algorithmic}
\end{algorithm}

Algorithms~\ref{alg:ippo}--\ref{alg:mappo} describe our implementation, with shaped rewards and capped PI feedback. Algorithm~\ref{alg:primal-dual} instead specifies the terminal-cost oracle and uncapped projected update analyzed in the proofs.


\section{Standard probability tools utilized}\label{app:probability-tools}
We record the precise versions of the standard probability results invoked below. These are textbook tools, not additional contributions of this paper, and their proofs are omitted. Conditional expectation, Markov's inequality, countable subadditivity, and the first Borel--Cantelli lemma are treated in \citet{durrett2019probability}; the bounded-variable exponential estimate is Hoeffding's lemma as treated in \citet{boucheron2013concentration}. Its conditional formulation below uses the same bounded-variable estimate after conditioning. All statements are on a probability space $(\Omega,\mathcal A,\mathbb P)$; identities and inequalities between conditional expectations are understood almost surely. They require no independence unless explicitly stated.

\begin{lemma}[Conditional expectation: tower, pull-out, and order properties]\label{lem:conditional-expectation}
Let $\mathcal H\subseteq\mathcal G\subseteq\mathcal A$ be sub-$\sigma$-algebras and let $X$ be an integrable real random variable. Then
\[
 \mathbb E[\mathbb E[X\mid\mathcal G]\mid\mathcal H]
 =\mathbb E[X\mid\mathcal H],
 \qquad \mathbb E[\mathbb E[X\mid\mathcal G]]=\mathbb E[X].
\]
If $Z$ is a bounded $\mathcal G$-measurable real random variable, then
\[
 \mathbb E[ZX\mid\mathcal G]=Z\mathbb E[X\mid\mathcal G].
\]
The pull-out identity also holds for a finite $\mathcal G$-measurable $Z$ when $X$ and $ZX$ are integrable. Conditional expectation is linear on integrable variables and order preserving: if integrable $X\leq Y$ almost surely, then $\mathbb E[X\mid\mathcal G]\leq\mathbb E[Y\mid\mathcal G]$ almost surely. In particular, deterministic bounds $a\leq X\leq b$ imply $a\leq\mathbb E[X\mid\mathcal G]\leq b$ almost surely. The tower and pull-out identities also hold, with extended nonnegative conditional expectations, for nonnegative variables and nonnegative measurable factors. See \citet{durrett2019probability}.
\end{lemma}

\begin{lemma}[Conditional Hoeffding bound]\label{lem:conditional-hoeffding}
Let $\mathcal G\subseteq\mathcal A$ be a sub-$\sigma$-algebra. Let $X$ be a real random variable with deterministic finite bounds $a\leq X\leq b$ almost surely, where $a\leq b$, and put $m=\mathbb E[X\mid\mathcal G]$. For every fixed $t\in\mathbb R$,
\[
 \mathbb E[\exp\{t(X-m)\}\mid\mathcal G]
 \leq\exp\!\left\{\frac{t^2(b-a)^2}{8}\right\}
 \quad\text{almost surely}.
\]
Equivalently,
\[
 \mathbb E[e^{tX}\mid\mathcal G]
 \leq\exp\!\left\{tm+\frac{t^2(b-a)^2}{8}\right\}.
\]
In particular, if $m=0$ and $|X|\leq c$ almost surely for a deterministic $c\geq0$, the bound is $\exp(t^2c^2/2)$. No independence between $X$ and $\mathcal G$ is assumed. The unconditional version results from choosing the trivial $\sigma$-algebra. See \citet{boucheron2013concentration} for Hoeffding's bounded-variable lemma and \citet{durrett2019probability} for conditional expectation.
\end{lemma}

\begin{lemma}[Markov's inequality and its exponential form]\label{lem:markov}
If $Y\geq0$ almost surely and $\mathbb E Y<\infty$, then, for every $u>0$,
\[
 \mathbb P(Y\geq u)\leq\frac{\mathbb E Y}{u}.
\]
Consequently, for a real random variable $X$, any $t>0$ such that $\mathbb E e^{tX}<\infty$, and any $x\in\mathbb R$,
\[
 \mathbb P(X\geq x)\leq e^{-tx}\mathbb E e^{tX}.
\]
The corresponding lower-tail bound follows by applying this statement to $-X$ when its exponential moment is finite. These inequalities also bound events with strict inequalities. See \citet{durrett2019probability}.
\end{lemma}

\begin{lemma}[Countable union bound and simultaneous events]\label{lem:union-bound}
For any sequence $(E_j)_{j\geq1}$ of events in $\mathcal A$,
\[
 \mathbb P\!\left(\bigcup_{j\geq1}E_j\right)
 \leq\sum_{j\geq1}\mathbb P(E_j).
\]
In particular, if $A_1,\ldots,A_m$ satisfy $\mathbb P(A_j)\geq1-\beta_j$ with $\beta_j\geq0$, then
\[
 \mathbb P\!\left(\bigcap_{j=1}^m A_j\right)
 \geq1-\sum_{j=1}^m\beta_j.
\]
A countable intersection of probability-one events has probability one. No independence is required. See \citet{durrett2019probability}.
\end{lemma}

\begin{lemma}[First Borel--Cantelli lemma]\label{lem:borel-cantelli}
Let $(E_m)_{m\geq1}$ be events in $\mathcal A$. If $\sum_{m=1}^{\infty}\mathbb P(E_m)<\infty$, then
\[
 \mathbb P(E_m\text{ infinitely often})
 =\mathbb P\!\left(\bigcap_{N=1}^{\infty}\bigcup_{m\geq N}E_m\right)=0.
\]
Equivalently, with probability one there is a finite, outcome-dependent index $N$ such that none of the events $E_m$ occurs for $m\geq N$. This direction of Borel--Cantelli requires no independence among the events. See \citet{durrett2019probability}.
\end{lemma}

\section{Experimental system and implementation}\label{app:protocol}

The shared proportional--integral controller combines the current measured budget excess with an accumulated contribution from past excesses. The current term responds to the latest episode, while the accumulated term maintains pressure when depletion repeatedly exceeds the budget. Its update is
\begin{equation}\label{eq:pi}
u_k=\widehat C_k-\varepsilon,\quad
I_{k+1}=\operatorname{clip}_{[0,I_{\max}]}(I_k+K_Iu_k),\quad
\lambda_{k+1}=[K_Pu_k+I_{k+1}]_+,
\end{equation}
Here $u_k$ is the measured budget excess and $I_k$ is its accumulated contribution to the price. The gains $K_I=0.03$ and $K_P=1$ set the strengths of the accumulated and current feedback, respectively; $I_{\max}=15$ caps the accumulated contribution. For actor updates, one-sided within-episode loss $g_t=[B_{t-1}-B_t]_+/K$ is attributed to IPPO agents in proportion to catch, while MAPPO receives the team penalty; the fixed scale is $w=2.5$, and PPO uses harvest divided by $K$.

Each episode evolves according to~\eqref{eq:stock}, with initial biomass $B_0=K$, continuous effort $e_i\in[0,1]$, and raw individual reward $R_i=h_i$. The horizon-$H$ terminal metric is $1-B_H/K$. Both actor families use PPO and discounted advantage estimation; IPPO uses local per-agent critics, and MAPPO uses a centralized critic and the joint reward. The experiment compares these learning methods under the same ecological budget.

\begin{table}[H]
\centering\small
\caption{Design of the constrained MARL experiments. Each method--budget--seed combination trains a separate controller for 20,000 epochs. The horizon is $H=60$ \editorialflag{for all experiments.}}\label{tab:protocol}
\begin{tabular}{ll}
\toprule Quantity & Value\\\midrule
Agents; initial stock; carrying capacity & $n=2$, $B_0=K=1000$\\
Growth; catchability; episode horizon & $r=0.3$, $q=0.5$, $H=60$\\
Actions & Continuous effort $e_i\in[0,1]$\\
Actor; critics & Squashed Gaussian; local (IPPO), central (MAPPO)\\
Learning objectives & Individual reward (IPPO); team reward (MAPPO)\\
Constraint feedback & Shared terminal-depletion PI multiplier\\
Discount; GAE parameter & $\gamma=0.99$, $\lambda_{\rm GAE}=0.95$\\
PPO passes; rollout batch & $5$; one complete episode ($60$ steps)\\
Actor learning rate; critic learning rate & $0.03$; $0.06$\\
PPO clipping; entropy coefficient & $0.2$; $0$\\
Latent Gaussian standard deviation & $[0.06,0.8]$\\
\editorialflag{Initial actor bias: IPPO; MAPPO} & $-1.5$; $-2.0$\\
PI gains; integral-state cap & $K_P=1$, $K_I=0.03$, $I_{\max}=15$\\
Shaping scale & $w=2.5$; catch-attributed (IPPO), team (MAPPO)\\
Evaluation reward & Unpenalized undiscounted harvest divided by $K$\\
\editorialflag{Full-run training length} & $20{,}000$ epochs\\
Seeds; constrained training runs & $3$ per method--budget; $2\times6\times3=36$\\
Depletion budgets & $\{0.01,0.05,0.1,0.3,0.4,0.6\}$\\\bottomrule
\end{tabular}
\end{table}

At each training epoch, the learner (1) samples an episode under the current joint actor; (2) records unpenalized returns divided by $K$, terminal depletion, and intermediate stock; (3) forms shaped PPO advantages, assigning IPPO's one-sided loss in proportion to catch or MAPPO's loss to the team; (4) updates the actors and critics; and (5) advances the common PI state by~\eqref{eq:pi}. Unpenalized normalized return and terminal cost, rather than shaped training reward, are used for evaluation. An idealized oracle iteration would instead solve~\eqref{eq:lagrangians} at each price and use~\eqref{eq:projected-dual}; these are distinct algorithms.

The main-text trajectories report the per-epoch mean and minimum--maximum range across three seeds for each method and budget. These displays summarize the realized training measurements.

\subsection{Reading the experimental plots}\label{app:plot-reading}

An epoch is a training episode/update index, whereas $t$ indexes physical steps within an episode of horizon $H=60$. The horizontal label ``Epoch (last 2000)'' in Figures~\ref{fig:convergence}--\ref{fig:lambda-main} resets the displayed coordinate to the start of the final 2,000-epoch window; zero on that axis is not the beginning of training. \editorialflag{Training lasts 20,000 epochs; the displayed axes cover only the final 2,000.}

In the plot labels, $T$ denotes the episode endpoint, corresponding to $H$ in the manuscript. Thus $B_T/K=B_H/K$ is terminal biomass as a fraction of carrying capacity $K$, and $c=1-B_T/K$ is dimensionless terminal depletion. Zero depletion means that the terminal stock equals carrying capacity; depletion one means an empty terminal stock. The budget $\varepsilon$ is the allowed depletion, equivalently requiring terminal stock ratio at least $1-\varepsilon$. A cost above its reference line is a violation in the displayed measurement. Neither the shaded seed range nor a single cost observation is a probability of collapse.

For the three main-text plots, panels (A) and (B) denote IPPO and MAPPO. The six legend entries are the budgets $\varepsilon=0.01,0.05,0.1,0.3,0.4,0.6$, shown respectively in dark blue, blue, cyan, green, orange, and red. At each displayed epoch, the central curve is the mean across the three seeds and the light envelope runs from their minimum to maximum. The title notation ``mean $\pm$ [min,max]'' denotes this range, not a standard deviation or confidence interval. In particular, the plotted $\bar c$ is an across-seed epochwise mean, $S^{-1}\sum_{s=1}^{S}c_{s,k}$ with $S=3$, whereas $\bar C_M$ in the theory averages expected costs across policy epochs. Dotted horizontal lines in the depletion plot are budget targets; the boundaries of the envelopes in the return and price plots are seed extrema.

The trainers log undiscounted harvest divided by carrying capacity. An individual return is $K^{-1}\sum_{t=0}^{H-1}h_{i,t}$, and both IPPO's joint return and MAPPO's team return are $K^{-1}\sum_{t=0}^{H-1}\sum_i h_{i,t}$. Their aggregate returns therefore measure the same quantity when the runs share $K$ and $H$. The label ``raw'' excludes the ecological penalty; it does not mean unnormalized biomass. The return can exceed one because harvest accumulates over an episode in which stock regenerates. An individual return and a team total are different aggregates. This reporting convention does not redefine the theoretical discounted value $J_i^\gamma$.

The vertical price axis is the nonnegative ecological PI multiplier $\lambda$ in~\eqref{eq:pi}, shared by both harvesters. It is unrelated to the advantage-estimation parameter $\lambda_{\rm GAE}=0.95$. Price magnitudes depend on the reward and penalty scales and are not resource fractions.

\section{Interpretation of the experimental curves}\label{app:sensitivity}\label{app:diagnostics}\label{app:tables}\label{app:dual}

The experimental evidence consists of Figures~\ref{fig:convergence}--\ref{fig:lambda-main}. Each depicts the final 2,000 training epochs for both methods and all six budgets. Figure~\ref{fig:convergence} shows strong stock preservation at the smallest budget, where the depletion curves lie below the reference line. At intermediate budgets, both the means and seed envelopes vary around the prescribed levels. These movements matter because a constraint is an upper bound: being near its reference line is not equivalent to remaining below it. The envelopes also show that a mean can conceal individual-seed exceedances. The displayed window supports an assessment of late-training regulation, not a full-training feasibility calculation. Conversely, individual exceedances do not disprove average-epoch feasibility, which permits compensation by slack epochs. The three-seed envelope measures observed variation, not uncertainty in an expected-cost estimate. 

Figure~\ref{fig:reward} shows low harvest return under the tightest budget and higher returns at larger budgets. The team curves for MAPPO at budgets $.4$ and $.6$ are close, indicating similar harvest returns at these two budgets. The return comparison must be read together with depletion: more harvest in a violating epoch does not establish a better feasible solution. Both panels use total unpenalized harvest divided by carrying capacity, making their aggregate rewards comparable in scale; neither panel displays the discounted theoretical objective.

Figure~\ref{fig:lambda-main} records continued feedback. IPPO's tightest-budget price remains positive while declining, whereas its budget-$.05$ price rises. MAPPO's intermediate-budget prices change direction during the window, and its budget-$.6$ price stays near zero. These are evolving controller outputs, not certified equilibrium multipliers. The different trajectories also prevent a universal monotone ordering of learned prices by budget.

For a fixed realized cost and integral state, increasing $\varepsilon$ decreases the error $u=\widehat C-\varepsilon$, hence weakens the PI update. Across separately trained runs, the policy, realized depletion, and integral history also change. A near-zero price is compatible with little sustained penalty pressure in the displayed regime, but does not certify feasibility over the whole training sequence. Likewise, a bounded price under the capped PI rule is not the multiplier-growth certificate of Proposition~\ref{prop:average-epoch-certificate}.

A monotone shadow-price interpretation belongs to an ideal cooperative optimum. For example, let $\mu_1,\mu_2$ be exact Lagrangian maximizers at prices $\lambda_1,\lambda_2\geq0$, respectively, each feasible and complementary at budgets $\varepsilon_1<\varepsilon_2$. Their two optimality inequalities, added together, give $(\lambda_1-\lambda_2)(C(\mu_2)-C(\mu_1))\geq0$. If $\lambda_2>0$, complementarity gives $C(\mu_2)=\varepsilon_2>\varepsilon_1\geq C(\mu_1)$, so $\lambda_1\geq\lambda_2$; if $\lambda_2=0$, the same ordering follows from nonnegativity. This argument assumes exact supporting prices and complementary feasible solutions. The observed PPO--PI trajectories satisfy no such certificate, and game equilibria use unilateral rather than joint optimality comparisons. The experimental price curves therefore diagnose continuing adaptation rather than establish a monotone shadow-price law.

\section{\editorialflag{Extended literature review}}\label{app:extended-context}

The economic lineage runs from open-access resource use~\citep{gordon1954economic} to strategic harvesting of a common stock~\citep{levhari1980fish}. Institutional research shows that communities can develop diverse arrangements for governing commons~\citep{ostrom1990governing,ostrom2010polycentric}; collapse is not an inevitable consequence of sharing a resource, nor is one centralized intervention a universal remedy. In MARL, sequential social dilemmas treat cooperation as a property of policies rather than single actions~\citep{leibo2017multiagent}. \citet{perolat2017commons} directly study common-pool resource appropriation by independently learning agents, relating sustainability to exclusion and inequality. That work provides a close environmental precedent; our emphasis is the explicit resource constraint and its policy-sequence guarantees. Inequity aversion~\citep{hughes2018inequity}, learned reciprocity~\citep{eccles2019learning}, and incentives supplied by other agents~\citep{yang2020incentivize} provide mechanisms for improving collective outcomes. We instead place the environmental requirement in an explicit constraint and adapt a shared price from measured budget excess. The learning procedure supplies the shared price that regulates depletion.

GovSim links AI-agent sustainability to communication, reasoning, and self-government~\citep{piatti2024collapse}. Our PPO agents have a different information and learning architecture, so that work motivates the shared-resource problem rather than serving as a numerical baseline. The broader distinction between agent welfare and external harm is illustrated by algorithmic collusion~\citep{calvano2020collusion} and organized by the multi-agent-risk taxonomy~\citep{hammond2025risks}. Learning models that omit explicit competitor behavior can also converge to different economic outcomes depending on assumptions and initial conditions~\citep{cooper2015pricing}; independent learning describes how agents train; it does not establish that their policies form an equilibrium. Finally, Melting Pot demonstrates why performance against familiar partners need not generalize to new populations~\citep{leibo2021melting}. Our study examines resource regulation in a fixed two-agent population.

A close environmental precedent is the renewable-fishery study of~\citet{danassis2022signals}. Decentralized harvesters observe a common periodic signal, enabling temporal conventions such as taking turns or leaving fallow periods. Their experiments connect policy coordination to sustainable harvesting. Our epoch-dependent policies also allocate resource use over time, with the sequence generated by depletion-price feedback and evaluated against an explicit average-epoch budget. The cooperative and strategic results relate solutions of the resulting unconstrained problems to feasibility, reward, and unilateral-deviation bounds. This comparison connects the ecological role of temporal coordination to the constrained solution concept analyzed here.

Methods based on state--action frequencies provide classical foundations for constrained MDPs and randomized stationary solutions~\citep{altman1999cmdp}. In convex optimization, averaging solutions of Lagrangian subproblems can recover feasible solutions~\citep{larsson1999ergodic}, with finite-iteration feasibility and objective bounds under Slater-type assumptions~\citep{nedic2009primal}. These results explain the mathematical role of averaging. In constrained RL, \citet{le2019batch} use a Lagrangian game and batch policy optimization to return a policy mixture, while \citet{miryoosefi2019convex} reduce convex constraints on expected measurements to repeated scalar-reward RL and return a mixture of the resulting policies. Thus both policy mixing and reduction to unconstrained RL have clear precedents. Our cooperative result develops bounds for the specified reset-episode terminal-cost construction; the strategic certificate additionally addresses unilateral deviations against fixed opponents' policy sequences. Long-term constrained online optimization similarly separates cumulative constraint control from per-round feasibility~\citep{mahdavi2012longterm}. Constrained actor--critic methods couple policy and multiplier learning through stochastic approximation~\citep{borkar2005actor}, while modern constrained RL gives strong-duality results under suitable assumptions~\citep{paternain2019duality}. State-augmented constrained RL treats evolving multipliers as part of policy execution~\citep{calvo2024state}. Our framework follows the game and cooperative epoch constructions of~\citet{das2025cmg,das2025lagrangian}, with a separately specified reset-episode value criterion. Unlike conditioning one learned actor on a multiplier, the ideal construction invokes an unconstrained oracle at each price and retains its returned policy.

When a player's admissible deviations depend on the other players, the relevant concept is generalized Nash equilibrium, with classical foundations in constrained games~\citep{rosen1965concave} and penalty methods for coupled feasible sets~\citep{facchinei2010penalty}. Classical constrained Markov-game theory gives sufficient conditions for stationary Nash equilibria~\citep{altman2000games}. \citet{chen2022correlated} instead develop a primal--dual method for a surrogate constrained correlated-equilibrium concept, with duality-gap and violation bounds. Its deviation tests differ from Definition~\ref{def:epoch-nash}: sharing a random policy index reproduces average values but does not by itself establish their equilibrium criterion. In the dynamic setting, constrained Markov potential games can fail strong duality~\citep{alatur2024constrained}; their independent-learning guarantees require additional structure and constraint qualifications~\citep{jordan2024independent}. We do not assume that the fishery is a potential game. The $\alpha$-potential approach of~\citet{das2024alphapotential} extends the structural perspective to games with bounded departures from the potential property. The present certificates do not assume a bound on departure from a potential; they instead assume that each unconstrained Lagrangian game is solved to within a specified bound on each player's gain from deviating. Consequently, they do not replace the algorithmic learning guarantees available under potential or approximate-potential structure. Our average-epoch deviation test also allows an agent to redistribute its constraint budget across epochs, a freedom excluded when every alternative policy must satisfy the budget at its own epoch. The resulting price-dispersion term makes explicit why a cooperative reward guarantee or a feasibility certificate alone does not establish strategic stability.

CPO studies constrained policy search with near-constraint satisfaction at individual iterations~\citep{achiam2017constrained}. Its multi-agent extensions include MACPO and MAPPO-Lagrangian for cooperative robot control~\citep{gu2023safe}. ACPO studies constrained average-reward MDPs~\citep{agnihotri2024acpo}, whose physical-time criterion differs from our average of discounted reset episodes. Predictive safety filters instead intervene on proposed actions to maintain state constraints~\citep{wabersich2021filter}; this addresses a different requirement from an average terminal-stock budget. Our practical actors use PPO, GAE, IPPO, and MAPPO~\citep{schulman2017proximal,schulman2016gae,dewitt2020ippo,yu2022surprising}, with PI feedback motivated by oscillation and overshoot in Lagrangian learning~\citep{stooke2020responsive}. Stable opponent shaping separately studies the stability of interacting learning dynamics~\citep{letcher2019stable}. Neither that result nor PPO clipping certifies the present learners' oracle accuracy or equilibrium. We assess their measured depletion and reward, and state separately the assumptions under which the ideal policy-sequence method has guarantees.

\end{document}